\documentclass[hidelinks,onefignum,onetabnum]{siamart250211}

\usepackage{lipsum}
\usepackage{amsfonts}
\usepackage{graphicx}
\usepackage{epstopdf}
\usepackage{algorithmic}
\usepackage{enumitem}
\usepackage{xcolor}
\usepackage{tikz}
\usepackage{float} 
\usepackage{multirow}
\usepackage{cite} 
\usetikzlibrary{arrows.meta, positioning, calc, shapes, fit, backgrounds}
\ifpdf
  \DeclareGraphicsExtensions{.eps,.pdf,.png,.jpg}
\else
  \DeclareGraphicsExtensions{.eps}
\fi

\newcommand{\bE}{\mathbf{E}}   
\newcommand{\bv}{\mathbf{v}}
\newcommand{\bx}{\mathbf{x}}
\newcommand{\bz}{\mathbf{z}}
\newcommand{\eps}{\varepsilon}

\newsiamremark{remark}{Remark}
\newsiamremark{hypothesis}{Hypothesis}
\crefname{hypothesis}{Hypothesis}{Hypotheses}
\newsiamthm{claim}{Claim}
\newsiamremark{fact}{Fact}
\crefname{fact}{Fact}{Facts}

\headers{Micro-Macro Tensor Neural Surrogates for Kinetic UQ}{W. Chen, G. Dimarco, and L. Pareschi}

\title{Micro-Macro Tensor Neural Surrogates for Uncertainty Quantification in Collisional Plasma\thanks{The work of WC was partially supported by the China Scholarship Council, and the National Natural Science Foundation of China (NSFC) through the Research Fund for the Doctoral Program (No. 125B2021). WC also acknowledges the hospitality of the University of Ferrara. 
  The work of GD was partially supported by the Italian Ministry of University and Research (MUR) through the PRIN
2020 project (No. 2020JLWP23) ``Integrated Mathematical Approaches to Socio–Epidemiological Dynamics”. The work of LP was partially supported by the Royal Society under the Wolfson Fellowship ``Uncertainty quantification, data-driven simulations and learning of multiscale complex systems governed by PDEs". LP also acknowledges the partial support by Fondo Italiano per la Scienza (FIS2023-01334) advanced grant "ADvanced numerical Approaches for MUltiscale Systems with uncertainties" - ADAMUS, and 
by the Italian Ministry of University and Research (MUR) through the PRIN 2022 project (No. 2022KKJP4X) ``Advanced numerical methods for time dependent parametric partial differential equations with applications". This work has been written within the activities of GNFM and GNCS groups of INdAM (Italian National Institute of High Mathematics).}}

\author{Wei Chen\thanks{School of Mathematical Sciences, Xiamen University, China 
  (\email{weichenmath@stu.xmu.edu.cn}).}
\and Giacomo Dimarco\thanks{Department of Mathematics and Computer Science, University of Ferrara, Italy 
  (\email{giacomo.dimarco@unife.it}).}
\and Lorenzo Pareschi\thanks{Maxwell Institute for Mathematical Sciences and Department of Mathematics, School of Mathematical and Computer Sciences, Heriot-Watt University, Edinburgh, UK
  (\email{l.pareschi@hw.ac.uk}). Also affiliated with Department of Mathematics and Computer Science, University of Ferrara, Italy.}
  }

\usepackage{amsopn}

\ifpdf
\hypersetup{
  pdftitle={Micro-Macro Tensor Neural Surrogates for Uncertainty Quantification in the Vlasov--Poisson--Landau System},
  pdfauthor={W. Chen, G. Dimarco, and L. Pareschi}
}
\fi

\begin{document}

\maketitle

\begin{abstract}
Plasma kinetic equations exhibit pronounced sensitivity to microscopic perturbations in model parameters and data, making reliable and efficient uncertainty quantification (UQ) essential for predictive simulations. 
However, the cost of uncertainty sampling, the high-dimensional phase space, and multiscale stiffness pose severe challenges to both computational efficiency and error control in traditional numerical methods. These difficulties are further amplified in the presence of collisions, where high-dimensional nonlocal collision integrals and conservation constraints impose severe additional restrictions. To overcome this, we present a variance-reduced Monte Carlo framework for UQ in the Vlasov--Poisson--Landau (VPL) system, in which neural network surrogates replace the multiple costly evaluations of the Landau collision term.
The method couples a high-fidelity, asymptotic-preserving VPL solver with inexpensive, strongly correlated surrogates based on the Vlasov--Poisson--Fokker--Planck (VPFP) and Euler--Poisson (EP) equations. 
For the surrogate models, we introduce a generalization of the separable physics-informed neural network (SPINN), developing a class of tensor neural networks based on an anisotropic micro-macro decomposition, to reduce velocity-moment costs, model complexity, and the curse of dimensionality. 
To further increase correlation with VPL, we calibrate the VPFP model and design an asymptotic-preserving SPINN whose small- and large-Knudsen limits recover the EP and VP systems, respectively.
Numerical experiments show substantial variance reduction over standard Monte Carlo, accurate statistics with far fewer high-fidelity samples, and lower wall-clock time, while maintaining robustness to stochastic dimension. 
\end{abstract}



\tableofcontents

\section{Introduction}
Kinetic equations play a central role in plasma physics, rarefied gas dynamics, and related multiscale problems, where macroscopic observables emerge from microscopic particle interactions~\cite{villani2002review,pareschi2013interacting,cercignani2013mathematical}. In many relevant applications, however, the predictive capability of kinetic simulations is severely affected by uncertainties in model parameters, initial data, and external forcing. These uncertainties may originate from incomplete physical knowledge, measurement errors, or intrinsic variability, and can be strongly amplified by nonlinear transport and collisional effects. As a result, reliable uncertainty quantification (UQ) has become an essential component of kinetic modeling.

Among the various kinetic descriptions, the Vlasov–Poisson–Landau (VPL) system is of paramount importance in plasma physics, as it models long-range Coulomb interactions through a nonlinear, nonlocal, and degenerate collisional operator. From the viewpoint of UQ, the VPL system presents a combination of major difficulties: a high-dimensional phase space, stiffness induced by collisional diffusion, conservation laws, and the high computational cost of evaluating the Landau operator~\cite{ZaJi2024, dimarco2014numerical}. While UQ methodologies for plasma kinetic equations have been extensively developed for simplified collision models, such as Bhatnagar--Gross--Krook and Fokker-Planck operators~\cite{AlbiDimarcoFerraresePareschi2025, MedagliaPareschiZanella2023, Frank2021, ZhuJin2017}, systematic UQ studies for the full VPL system remain scarce~\cite{MedagliaPareschiZanella2024,hu2016stochasticvpl,bailo2025uncertainty}.

Two principal classes of methods have been explored for UQ in kinetic equations. Intrusive approaches, such as stochastic Galerkin methods, can achieve spectral convergence when the solution depends smoothly on the random inputs, but their cost grows rapidly with the stochastic dimension and they require substantial modifications of deterministic solvers~\cite{hu2016stochastic,liu2024spectral,liu2020bi,shu2017stochastic,xiu2010numerical}. Nonintrusive Monte Carlo (MC) methods, by contrast, are solver-independent and trivially parallelizable, but their slow convergence rate makes high-accuracy simulations prohibitively expensive in complex kinetic regimes~\cite{caflisch1998monte,fairbanks2017low,giles2015multilevel,hu2021uncertainty,mishra2012multi,peherstorfer2018convergence,peherstorfer2016optimal,DimarcoLiuPareschiZhu2025}. 

In recent years, variance-reduction techniques based on multifidelity~\cite{dimarco2019multi,dimarco2020multiscale,peherstorfer2016optimal} and multilevel~\cite{giles2015multilevel,hu2021uncertainty} ideas have emerged as a powerful strategy to alleviate the limitations of Monte Carlo sampling. The core idea is to combine a small number of high-fidelity samples with a large number of inexpensive, strongly correlated low-fidelity evaluations, yielding unbiased estimators with significantly reduced variance. For kinetic equations, multiscale control-variate methods have proven particularly effective when suitable reduced models are available. However, the construction of accurate and robust low-fidelity surrogates for collisional plasma models remains a major challenge.

In this work, following our previous approach for the Boltzmann equation~\cite{chen2025spnn}, we develop a variance-reduced Monte Carlo framework for uncertainty quantification in the Vlasov–Poisson–Landau system, in which low-fidelity surrogates are generated by structure-preserving neural networks. The high-fidelity model is discretized using an asymptotic-preserving solver for the VPL equation, while the low-fidelity models are based on the Vlasov–Poisson–Fokker–Planck (VPFP) and Euler–Poisson (EP) systems. These reduced descriptions capture, respectively, the collisional and hydrodynamic limits of the VPL dynamics and serve as control variates within the variance-reduction strategy. 
Although VPFP and EP are significantly cheaper than VPL on a per-sample basis, directly solving these PDEs for the $\mathcal{O}(10^4)$ realizations required by variance-reduced Monte Carlo would still be extremely expensive. Neural-network surrogates are therefore introduced not as replacements for the physical models, but as a means to reduce their cost and enable massive sampling at inference time. We mention also operator learning approaches for the Landau equation~\cite{oPINN2023} and refer to~\cite{GyroSwim2025, plasmaDD2022} for data driven construction of plasma models.

A key novelty of the present work lies in the architecture of the neural surrogate, which differs fundamentally from our previous PINN-based approach. Rather than approximating the full distribution function through a single neural network, we adopt a composite micro–macro architecture, inspired by the underlying structure of collisional kinetic equations. Specifically, the macroscopic component of the solution is learned by an anisotropic Maxwellian through a dedicated neural network that directly parameterizes the moment variables. The microscopic component, representing deviations from the anisotropic Maxwellian, is instead approximated by a separate tensor neural network~\cite{JCM,oh2025separable} acting on velocity space. By exploiting the low-rank structure of kinetic solutions in velocity space, tensor neural networks drastically reduce the cost of velocity-moment evaluations and mitigate the curse of dimensionality.
The two components are coupled through the same asymptotic-preserving loss function, yielding a structured representation of the kinetic solution.

The architecture adopts a hierarchical micro-macro structure, with a model-driven decomposition embedded into the architecture itself. Compared with single-network architectures, this composite approach leads to more efficient training, improved generalization across Knudsen regimes, and stronger correlation with the high-fidelity VPL solution, a crucial requirement for effective variance reduction.

The construction of the present method relies on several building blocks, which are summarized below.

\begin{itemize}
\item \textit{Multifidelity design.}
The VPL system is treated as the high-fidelity model, while the VPFP and EP systems are employed as low-fidelity surrogates in the training. Within the variance-reduced Monte Carlo (VRMC) framework, optimal control-variate weights are deduced in order to minimize the estimator variance associated with the VPL solution.

\item \textit{High-fidelity solver.}
For the VPL system, we adopt a fast spectral solver of the Landau collision operator~\cite{pareschi2000fast}, penalized by a Fokker–Planck collision operator~\cite{dimarco2015numerical,jin2011class}. Time integration is performed using high-order implicit–explicit (IMEX) Runge–Kutta schemes~\cite{pareschi2000implicit}, combined with a weighted essentially non-oscillatory (WENO) spatial discretization~\cite{shu2020essentially}. The resulting method is asymptotic preserving (AP), with a stable time step independent of the Knudsen number, thus enabling the efficient generation of accurate high-fidelity samples.

\item \textit{Low-fidelity surrogates via neural networks.}
Building on~\cite{chen2025spnn}, we introduce separable physics-informed neural networks (SPINN), parameterized by the random inputs, to generate low-fidelity samples for the VPFP and EP equations at low computational cost. The SPINN decomposition is further optimized by adopting an anisotropic Maxwellian distribution and tensor neural networks, which enhances the approximation capability. The resulting surrogate models drastically reduce the cost of velocity-moment evaluations, lower model complexity, alleviate the curse of dimensionality associated with high-dimensional parameters, and accelerate the training process.

\item \textit{Model calibration and asymptotic-preserving property.}
For the VPFP surrogate, a calibration parameter is introduced using a subset of VPL data in order to reduce the model discrepancy and enhance variance reduction. In addition, a windowed training strategy is employed to improve long-time accuracy and reduce error accumulation. Finally, the surrogate model is constructed to be asymptotic preserving, consistently recovering the Euler-Poisson (EP) fluid limit and the collisionless Vlasov-Poisson (VP) regime in the respective highly collisional and mean-free-path dominated limits.\end{itemize}

The remainder of the paper is organized as follows.
Section~\ref{mf} introduces the VPL, VPFP, and EP models that form the multiscale hierarchy underlying the proposed approach.
Section~\ref{dm} describes the deterministic asymptotic-preserving numerical method adopted for the VPL system.
Section~\ref{nn} presents the tensor neural network architecture and develops the UQ-SPINN framework tailored to the VPFP and EP equations.
Section~\ref{vrmc} introduces the variance-reduced Monte Carlo methodology based on the neural-network surrogates for uncertainty quantification.
Section~\ref{ne} reports a set of numerical experiments that validate the proposed framework and illustrate its effectiveness across different regimes.
Finally, the last section summarizes the main conclusions and outlines perspectives for future work.

\section{Model Formulation}
\label{mf}
The Vlasov--Poisson--Landau (VPL) equation is widely used in kinetic models for describing long-range Coulomb interactions within weakly ionized gases. 
It is represented by a nonlinear partial integro-differential equation of the form
\begin{equation}
\label{system}
  {\partial_t}f + \bv \cdot \nabla_\bx f + \bE \cdot \nabla_\bv f = \frac{1}{\varepsilon} Q(f, f), \quad \bv \in \mathbb{R}^{d_v}, \quad \bx \in \Omega \subset \mathbb{R}^{d_x},
\end{equation}
where $f(t, \bx, \bv)$ denotes the distribution function, which depends on time $t$, spatial location $\bx$, and particle velocity $\bv$. 
Here, $\eps$ is the Knudsen number, and $\bE(t, \bx)$ is the self-consistent electric field determined by the solution to the normalized Poisson equation
\begin{equation}
\label{poisson}
  \bE(t, \bx) = -\nabla_\bx \phi(t, \bx), \qquad \Delta_\bx \phi(t, \bx) = 1 - \int_{\mathbb{R}^{d_v}} f(t, \bx, \bv)\,d\bv,
\end{equation}
while $Q(f, f)$ refers to the Landau collision operator, given by
\begin{equation*}
Q(f, f)(\bv) = \nabla_\bv \cdot \int_{\mathbb{R}^{d_v}} \Phi(\bv-\bv^*) \left[ \nabla_\bv f(\bv) f(\bv^*) - \nabla_{\bv^*} f(\bv^*) f(\bv) \right] d\bv^*.
\end{equation*}
In this definition, the dependence on $\bx$ and $t$ is omitted for simplicity, and $\Phi$ is a symmetric, non-negative matrix which encodes the specifics of the particle interaction.
For example,
\begin{equation*}
  \Phi(\bv) = |\bv|^{\gamma+2} S(\bv), \quad \gamma \in \mathbb{R}, \quad {\rm and} \quad S(\bv) = \mathbf{I} - \frac{\bv \otimes \bv}{|\bv|^2},
\end{equation*}
where $\mathbf{I}$ is the identity matrix.
The value of $\gamma$ determines the classification of interactions: hard potentials correspond to $\gamma > 0$, Maxwellian molecules to $\gamma = 0$, and soft potentials to $\gamma < 0$. 
The most important scenario for plasma physics is the Coulomb interaction with $\gamma = -3$.

The Landau collision operator shares similarities with the classical Boltzmann collision integral for rarefied gases. 
Kinetic theory ensures that the operator conserves key physical quantities, specifically mass, momentum, and energy:
\begin{equation}
\label{conservation}
  \int_{\mathbb{R}^{d_v}} Q(f, f)(\bv) \;
  \psi(\bv) d\bv = 0,
\end{equation}
where $\psi(\bv) = (1, \bv, {\vert \bv \vert}^2/2)^{\top}$ is the collision invariants.
Multiplying the Landau equation by the test function $\psi(\bv)$ and integrating over velocity space yields the moment equations:
\begin{equation}
\label{moment}
  \partial_t \int_{\mathbb{R}^{d_v}} f \psi(\bv) d\bv + \nabla_{\bx} \cdot \int_{\mathbb{R}^{d_v}} \bv f \psi(\bv) d\bv 
   = \bE \cdot  \int_{\mathbb{R}^{d_v}}f \nabla_\bv \psi(\bv)  d\bv,
\end{equation}
where we have used integration by parts.
Moreover, the entropy production is non-negative:
\begin{equation*}
  \frac{d H(t)}{dt} = - \frac{d}{dt} \int_{\mathbb{R}^{d_v}} f(t, \bv) \ln (f(t, \bv))\,d\bv \ge 0.
\end{equation*}
This implies that the equilibrium states of the Landau operator, i.e., functions $f$ such that $Q(f, f) = 0$, are given by local Maxwellian distributions:
\begin{equation}
\label{maxwellian0}
  \mathcal{M}(\mathbf{U}) = \frac{\rho}{(2\pi  T)^{d_v/2}} \exp \left( -\frac{|\bv - \mathbf{u}|^2}{2 T} \right),
\end{equation}
with
\begin{equation*}
  \mathbf{U} = \left(\rho, \rho \mathbf{u}, E \right)^{\top}, \quad E = \frac12\rho \left(\vert \mathbf{u} \vert^2 + d_v  T\right).
\end{equation*}
Here $E$ is the total energy, $\rho$ is the total mass, $\mathbf{u}$ is the mean velocity, and $T$ denotes the temperature, determined by
\begin{equation*}
  \rho = \int_{\mathbb{R}^{d_v}} f(\bv)\,d\bv, \qquad
  \mathbf{u} = \frac{1}{\rho} \int_{\mathbb{R}^{d_v}} f(\bv) \bv\,d\bv, \qquad
  T = \frac{1}{{d_v}  \rho} \int_{\mathbb{R}^{d_v}} f(\bv) |\mathbf{u} - \bv|^2 d\bv.
\end{equation*}
Owing to its sharing the same equilibrium, conservation properties, and diffusive character with the Landau operator—and to its ease of inversion—the Fokker--Planck (FP) collision operator is commonly employed as a penalization term in asymptotic preserving (AP) schemes for the Landau equation~\cite{jin2011class,dimarco2015numerical}. 
Its form is given by
\begin{equation}
\label{FPcollision}
  P(f) = \nabla_{\bv} \cdot \left(\mathcal{M} \nabla_{\bv} \left( \frac{f}{\mathcal{M}} \right) \right),
\end{equation}
and the Vlasov--Poisson--Fokker--Planck (VPFP) equation can be regarded as an approximate model for VPL:
\begin{equation}
\label{VPFP}
  {\partial_t}f + \bv \cdot \nabla_\bx f + \bE \cdot \nabla_\bv f = \frac{\mu}{\varepsilon} P(f), \quad \bv \in \mathbb{R}^{d_v}, \quad \bx \in \Omega \subset \mathbb{R}^{d_x},
\end{equation}
where $\mu > 0$ denotes the collision frequency.

In the limit of vanishing Knudsen number ($\varepsilon\to 0$), both the VPFP and VPL systems formally reduce to the closed compressible Euler--Poisson (EP) system, obtained via the Poisson relation \eqref{poisson} and the moment equations \eqref{moment}~\cite{dimarco2014numerical}:
\begin{equation}
\label{EP}
\begin{aligned}
  \partial_t \mathbf{U} + \nabla_{\bx} \cdot \mathbf{F} (\mathbf{U}) = \mathbf{S}(\mathbf{U}, \phi),
\end{aligned}
\end{equation}
with
\begin{equation*}
\label{3eq}
\left\{
\begin{aligned}
  &\Delta_\bx \phi = 1 - \rho,\\
  &\mathbf{F} (\mathbf{U}) = \left(\rho \mathbf{u}, \rho \mathbf{u} \otimes \mathbf{u} + p \mathbf{I}, \left(E + p\right) \mathbf{u} \right)^{\top}, \\
  &\mathbf{S}(\mathbf{U}, \phi) = \left(0, -\rho \nabla_{\bx} \phi, -\rho \mathbf{u}  \cdot\nabla_{\bx} \phi \right)^{\top},
\end{aligned}
\right.
\end{equation*}
where $p = \rho  T$ denotes the pressure.

\section{Deterministic Methods}
\label{dm}
The Landau collision operator is of diffusive type: its spectrum is unbounded in the continuum and, upon discretization, exhibits eigenvalues of order $\mathcal{O}(1/\Delta v^{2})$. 
This induces the characteristic parabolic stiffness, forcing explicit schemes to obey the severe time-step restriction $\Delta t \sim (\Delta v)^{2}$. 
Although standard implicit methods enjoy larger stability regions, the high-dimensional, nonlocal coupling of the discretized operator $Q$ renders its inversion prohibitively expensive in practice. 
To alleviate this difficulty, a penalization operator which is easier to invert should be introduced. 
The Fokker-Planck (FP) operator \eqref{FPcollision} shares the Maxwellian equilibrium, conservation laws, and diffusive character of the Landau operator while possessing a structure more amenable to inversion~\cite{dimarco2015numerical,jin2011class}. 
Penalizing $Q$ with the FP operator $P$ replaces (or absorbs) the stiff component, relaxes the time-step constraint, and enables the design of high-order AP schemes that are uniform in the limit $\varepsilon \to 0$.
Therefore, the original VPL system \eqref{system} can be modified in
\begin{equation*}
  {\partial_t}f + \mathbf v \cdot \nabla_\bx f + \bE(t, \bx) \cdot \nabla_\bv f = \underbrace{\frac{1}{\eps} \left(Q(f, f) - \beta P(f)\right)}_{\text{non-stiff}} + \overbrace{\frac{1}{\eps} \beta P(f)}^{\text{stiff}}
\end{equation*}
and $\beta$ is a large enough constant, some details can be found in~\cite{jin2011class}.

To obtain high-order accuracy in time, we adopt an $s$-stage IMEX Runge-Kutta (RK) method~\cite{BoscarinoPareschiRusso2024}, which can be represented by a double $tableau$
\begin{equation*}
  \begin{array}{c|c}
  \tilde{c} & \tilde{A}\\
  \hline
  \vspace{-0.25cm}
  \\
  & \tilde{b}^T \end{array} \ \ \ \ \ \qquad
  \begin{array}{c|c}
  {c} & {A}\\
  \hline
  \vspace{-0.25cm}
  \\
  & {b^{\top}} \end{array},
\end{equation*}
where both $\tilde{A} = (\tilde{a}_{ij})$ and $A = ({a}_{ij})$ are $s \times s$ matrices.
The matrix $\tilde{A}$ which is lower triangular with a zero diagonal and another lower triangular matrix $A$ with a non-zero diagonal are applied for explicit and implicit parts, respectively. 
We note that a diagonally implicit RK scheme with a lower triangular matrix $A$ guarantees an efficient implementation of such high order IMEX multi-stage RK schemes. 
Here $\tilde{c}=(\tilde{c}_1,...,\tilde{c}_s)^{\top}$, $\tilde{b}=(\tilde{b}_1,...,\tilde{b}_s)^{\top}$,  $c=(c_1,...,c_s)^{\top}$, and $b=(b_1,...,b_s)^{\top}$ are $s \times 1$ vectors. 
The coefficients $\tilde{c}$ and $c$ satisfy the typical assumptions
\begin{eqnarray*}
    \tilde{c}_i = \sum_{j=1}^{i-1} \tilde a_{ij}, \ \ \ c_i = \sum_{j=1}^{i} a_{ij}.
\end{eqnarray*}
Later, for the consideration of AP property, we consider the globally stiff accuracy property, i.e., the implicit part of the
Butcher table satisfies the condition $b^{\top} = e_s^{\top}A = \tilde{b}^{\top} = e_s^{\top} \tilde{A}$, with $e_s^{\top} = (0, \cdots 0, 1)^{\top}$.

To make the IMEX discretization precise, we first rewrite the semi-discrete VPL equation in the form
\begin{equation*}
    \partial_t f = \mathcal{F}(f) + \mathcal{G}(f),
\end{equation*}
where the non-stiff part $\mathcal{F}(f)$ is treated explicitly and the stiff part $\mathcal{G}(f)$ is treated implicitly. 
In the present work, we take
\begin{equation*}
    \mathcal{F}(f)
    =
    - \mathbf v\cdot \nabla_{\mathbf x} f
    - \mathbf{E}(t, \mathbf x)\cdot \nabla_{\mathbf v} f
    + \frac{1}{\varepsilon}\bigl(Q(f,f)-\beta P(f)\bigr),
    \qquad
    \mathcal{G}(f)
    =
    \frac{1}{\varepsilon}\beta P(f).
\end{equation*}
The corresponding \(s\)-stage IMEX Runge--Kutta scheme is given by
\begin{equation*}
    f^{(k)}
    =
    f^n
    + \Delta t
    \sum_{j=1}^{k-1} \tilde{a}_{kj}\mathcal{F}\bigl(f^{(j)}\bigr)
    + \Delta t
    \sum_{j=1}^{k} a_{kj}\mathcal{G}\bigl(f^{(j)}\bigr),
    \qquad k=1,\ldots,s,
    \label{eq:imex-stage}
\end{equation*}
and
\begin{equation*}
    f^{n+1}
    =
    f^n
    + \Delta t
    \sum_{k=1}^{s} \tilde{b}_{k}\mathcal{F}\bigl(f^{(k)}\bigr)
    + \Delta t
    \sum_{k=1}^{s} b_{k}\mathcal{G}\bigl(f^{(k)}\bigr) = f^{(s)},
    \label{eq:imex-update}
\end{equation*}
with the globally stiff accuracy property $b^{\top} = e_s^{\top}A$ and $\tilde{b}^{\top} = e_s^{\top} \tilde{A}$.
Since \(\tilde{A}\) is strictly lower triangular, the contribution from
\(\mathcal{F}\) is explicit. Since \(A\) is lower triangular with a
nonzero diagonal, the stiff contribution from \(\mathcal{G}\) is treated
implicitly at each stage.

The updating process of the solution is given as follows:
\begin{itemize}
  \item \textbf{Step 1.} Starting from $f^n$ at time $t^n$, for an intermediate stage $k (1 \leq k \leq s)$, we first compute $Q \left( f^{(k - 1)} \right)$ and $\bE^{(k - 1)}$ from $f^{k - 1}$.

  \item \textbf{Step 2.} Updating the advection term and electric field term:
  \begin{equation*}
    \tilde{f}^{(k)} = f^n - \sum_{i = 1}^{k - 1} \tilde{a}_{k i} \Delta t \left( \bv \cdot \nabla_{\bx} f^{(i)} + \bE^{(i)} \cdot \nabla_{\bv} f^{(i)} \right).
  \end{equation*}

  \item \textbf{Step 3.} Computing $\mathcal{M}(f^{(k)})$ with the conservation property \eqref{conservation}, i.e., $\mathcal{M}(f^{(k)}) = \mathcal{M}(\tilde{f}^{(k)})$.

  \item \textbf{Step 4.} Solving the following system
  \begin{equation*}
    \left(\eps - a_{kk} \Delta t \beta P\right) \left(f^{(k)}\right) = \eps \bar{f}^{(k)}
  \end{equation*}
  with the expression
  \begin{equation*}
    \eps \bar{f}^{(k)} = \eps \tilde{f}^{(k)} + \sum_{i = 1}^{k - 1} \tilde{a}_{k i} \Delta t \left( Q \left( f^{(i)} \right) - \beta P \left( f^{(i)} \right) \right) + \sum_{i = 1}^{k - 1} a_{k i} \Delta t \beta P \left( f^{(i)} \right).
  \end{equation*}
  to get $f^{(k)}$ and then obtain $P\left({f^{(k)}}\right)$.

  \item \textbf{Step 5.} Updating the solution $f^{n + 1} = f^{(s)}$.
\end{itemize}

To balance computational efficiency and accuracy for the physical and velocity space derivative terms indicated in the flowchart, we employ a finite difference weighted-essentially non-oscillatory scheme~\cite{shu2020essentially}. 
Moreover, the Landau collision operator is evaluated using a fast spectral method to enhance computational efficiency~\cite{pareschi2000fast}. Note that, further computational efficiency can be achieved using IMEX linear multistep methods~\cite{DiPa2017}.

\section{Neural Network}
\label{nn}
Solutions to kinetic equations often concentrate near a low-dimensional manifold and admit a low-rank structure~\cite{dektor2021rank}. 
Leveraging this property, we construct a tensor neural network~\cite{JCM,oh2025separable} to alleviate the curse of dimensionality by curbing the exponential growth in memory usage and computational cost.
Fig.~\ref{F1} depicts the tensor neural network $\mathcal{T}$ comprising $r$ distinct subnetworks. 
The input coordinate along the $i$-th dimension is denoted $v_i$, and dashed boxes indicate hidden layers. 
Each subnetwork yields $p$ rank-1 components. 
The overall output of $\mathcal{T}$ is obtained by taking elementwise (Hadamard) products across the $r$ subnetwork outputs and then summing over the $p$ components.

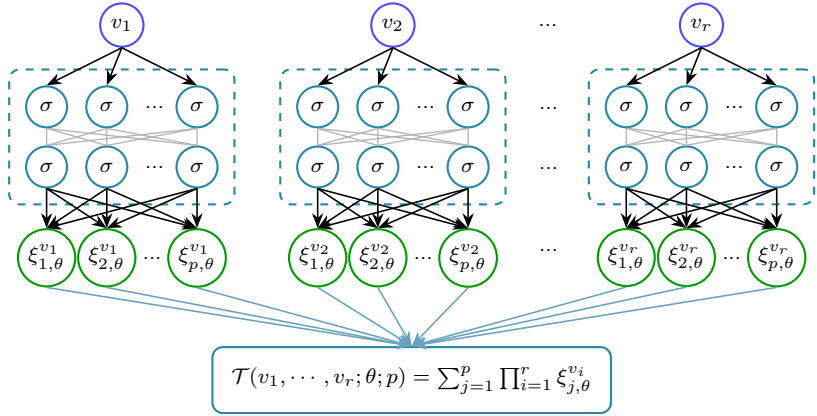
\begin{figure}[htb]
    \centering
  \begin{tikzpicture}[x=1.0cm, y=1.0cm, >=Stealth, 
      every node/.style={font=\sffamily},
      every path/.style={line width=1pt},
      inputnode/.style={draw=blue!70, thick, circle, minimum size=0.3cm, font=\sffamily\footnotesize},
      hiddencircle/.style={draw=cyan!60!black, thick, circle, minimum size=0.3cm, font=\sffamily\footnotesize},
      outputnode/.style={draw=green!60!black, thick, circle, minimum size=0.3cm, font=\sffamily\footnotesize},
      arrow1/.style={-Stealth, line width = 0.6pt},
      arrow2/.style={-, line width = 0.5pt},
      dashbox/.style={draw=cyan!60!black, rounded corners, dashed, thick, inner sep=6pt, fill=none},
      rect/.style={draw=cyan!60!black, rounded corners, thick, font=\sffamily\footnotesize}
      ]

  \node[inputnode, inner sep=2.5pt] (a1) {$v_1$};
  \node[inputnode, inner sep=2.5pt, right = 3cm of a1] (a2) {$v_2$};
  \node[right = 1.5cm of a2, font=\sffamily\footnotesize] (cdot) {...};
  \node[below = 0.65cm of a2, xshift={2.07cm}, font=\sffamily\footnotesize] {...};
  \node[below = 1.45cm of a2, xshift={2.07cm}, font=\sffamily\footnotesize] {...};
  \node[below = 2.55cm of a2, xshift={2.07cm}, font=\sffamily\footnotesize] {...};
  \node[inputnode, inner sep=2.5pt, right = 3.5cm of a2] (a3) {$v_r$};

  \foreach \j in {1,2,3}
    \foreach \i in {1,3}{
      \node[hiddencircle, below=0.5cm of a\j, xshift={(\i-2)*1cm}] (h1\i\j) {$\sigma$};
      \node[hiddencircle, below=1.3cm of a\j, xshift={(\i-2)*1cm}] (h2\i\j) {$\sigma$};
      }

  \foreach \j in {1,2,3}{
    \node[hiddencircle, below=0.5cm of a\j, xshift={-0.2cm}] (h12\j) {$\sigma$};
    \node[hiddencircle, below=1.3cm of a\j, xshift={-0.2cm}] (h22\j) {$\sigma$};
    }

  \foreach \j in {1,2,3}
    \foreach \i in {1,2}{
      \node[right = 0.09cm of h\i2\j, font=\sffamily\footnotesize] {...};
      }

  \foreach \j in {1,2,3}
    \node[fit=(h11\j) (h23\j), dashbox ] (boxnn) {};

  \node[outputnode, inner sep=1.5pt, below=2.4cm of a1, xshift={-1cm}] (xi11) {$\xi^{v_1}_{1, \theta}$};
  \node[outputnode, inner sep=1.5pt, below=2.4cm of a1, xshift={-0.2cm}] (xi12) {$\xi^{v_1}_{2, \theta}$};
  \node[outputnode, inner sep=1.5pt, below=2.4cm of a1, xshift={1cm}] (xi13) {$\xi^{v_1}_{p, \theta}$};

  \node[outputnode, inner sep=1.5pt, below=2.4cm of a2, xshift={-1cm}] (xi21) {$\xi^{v_2}_{1, \theta}$};
  \node[outputnode, inner sep=1.5pt, below=2.4cm of a2, xshift={-0.2cm}] (xi22) {$\xi^{v_2}_{2, \theta}$};
  \node[outputnode, inner sep=1.5pt, below=2.4cm of a2, xshift={1cm}] (xi23) {$\xi^{v_2}_{p, \theta}$};

  \node[outputnode, inner sep=1.5pt, below=2.4cm of a3, xshift={-1cm}] (xi31) {$\xi^{v_r}_{1, \theta}$};
  \node[outputnode, inner sep=1.5pt, below=2.4cm of a3, xshift={-0.2cm}] (xi32) {$\xi^{v_r}_{2, \theta}$};
  \node[outputnode, inner sep=1.5pt, below=2.4cm of a3, xshift={1cm}] (xi33) {$\xi^{v_r}_{p, \theta}$};

  \foreach \j in {1,2,3}
    \node[right = -0.05cm of xi\j2, font=\sffamily\footnotesize] {...};

  \node[below = 4.1cm of a2, xshift = 0.25cm, font=\sffamily\footnotesize] (final) {$\mathcal{T}(v_1,\cdots,v_{r};\theta;p) = \sum_{j=1}^{p} \prod_{i=1}^r \xi^{v_i}_{j, \theta}$};
  \node[fit=(final), rect] (boxnn) {};

  \foreach \i in {1,2,3} 
    \foreach \j in {1,2,3}{
      \draw[arrow1, black!100] (a\j.south) -- (h1\i\j.north);
  }

  \foreach \i in {1,2,3} 
    \foreach \j in {1,2,3}
      \foreach \k in {1,2,3} {
        \draw[arrow2, gray!60] (h1\j\i.south) -- (h2\k\i.north);
    }

  \foreach \i in {1,2,3}
    \foreach \j in {1,2,3}
      \foreach \k in {1,2,3}
        \draw[arrow1, black!100] (h2\j\i.south) -- (xi\i\k.north);

  \foreach \i in {1,2,3}{
    \draw[arrow1, cyan!40!gray] (xi1\i.south) -- (boxnn.north);
    \draw[arrow1, cyan!40!gray] (xi2\i.south) -- (boxnn.north);
    \draw[arrow1, cyan!40!gray] (xi3\i.south) -- (boxnn.north);
  }

  \end{tikzpicture}

  \caption{The Tensor Neural Network.}
    \label{F1}
\end{figure}

Tensor neural networks are particularly amenable to integration, as they reduce an $r-$fold integral to a product of $r$ one-dimensional integrals, thereby mitigating the curse of dimensionality in numerical integration. 
For illustration, we consider $d_v = r = 2$, $\bv=(v_1,v_2)^{\top}$, and define the one-dimensional moments
\begin{equation*}
  \eta_{j,\theta}^{l, k} := \int_{\mathbb{R}} \xi_{j,\theta}^{v_l} v_l^k d v_l, \quad k=0,1,2 \quad {\rm and} \quad l = 1, 2.
\end{equation*}
Then the two-dimensional integral reduces to a vector of products of these moments:
\begin{equation}
\label{int}
\begin{aligned}
  & \int_{\mathbb{R}^{d_v}} \mathcal{T}(v_1, v_2 ;\theta;p)  \psi(\bv) d \bv \\
  = & \sum_{j=1}^{p} \int_{\mathbb{R}^{d_v}} \prod_{i = 1}^2 \xi_{j,\theta}^{v_i} \left(1, \bv, \vert \bv \vert^2\right)^{\top} d \bv\\
  = & \sum_{j=1}^{p} \left(
  \prod_{i = 1}^2 \eta_{j,\theta}^{i, 0}, 
  \eta_{j,\theta}^{1, 1} \eta_{j,\theta}^{2, 0},
  \eta_{j,\theta}^{1, 0} \eta_{j,\theta}^{2, 1},
  \frac12 \left( \eta_{j,\theta}^{1, 2}\eta_{j,\theta}^{2, 0} + \eta_{j,\theta}^{1, 0}\eta_{j,\theta}^{2, 2} \right)
  \right)^{\top},
\end{aligned}
\end{equation}
which substantially reduces the impact of the curse of dimensionality on computational efficiency.

Building on~\cite{oh2025separable}, we construct a separable physics-informed neural network (UQ-SPINN) for the VPFP equation comprising three tensor neural networks, $\widetilde{\mathbf{W}}, g$, and $\phi$ (Fig.~\ref{F3}).
Here, $\mathcal{T}$ denotes a tensor neural network that takes uncertainty $\mathbf{z}$, time $t$, physical space $\mathbf{x}$, and velocity $\mathbf{v}$ as inputs. 
Here, the weight matrix $\mathbf{w}$ is an $(1+2d_v)\times p$ matrix learned by training a subnetwork and requires no explicit inputs. 
The output of network $\widetilde{\mathbf{W}}$  is a tensor comprising the density, $d_v$ velocity components, and $d_v$ directional temperatures, 
i.e.,
\[
    \widetilde{\mathbf W}
    =
    \left(
    \widetilde{\rho},
    \widetilde{u}_1,\ldots,\widetilde{u}_{d_v},
    \widetilde{T}_1,\ldots,\widetilde{T}_{d_v}
    \right)^{\top}.
\]
The component index is ordered as
\[
    \widetilde{\mathbf W}_1=\widetilde{\rho}, \qquad
    \widetilde{\mathbf W}_{1+i}=\widetilde{u}_i,\quad i=1,\ldots,d_v,
    \qquad
    \widetilde{\mathbf W}_{1+d_v+i}=\widetilde{T}_i,\quad i=1,\ldots,d_v,
\]
and the $l$-th component of $\widetilde{\mathbf{W}}$ is given by
\begin{equation*}
\widetilde{\mathbf{W}}_{l}
= \sum_{j=1}^{p} \mathbf{w}_{j}^{\,l}\,
\xi^{z_{1}}_{j,\theta}\cdots\xi^{z_{d_z}}_{j,\theta}\,
\xi^{t}_{j,\theta}\,
\xi^{x_{1}}_{j,\theta}\cdots\xi^{x_{d_x}}_{j,\theta}.
\end{equation*}

\begin{figure}[htb]
    \centering
  \begin{tikzpicture}[x=1.0cm, y=1.0cm, >=Stealth, 
      every node/.style={font=\sffamily},
      every path/.style={line width=1pt},
      circle1/.style={draw=green!60!black, thick, circle, minimum size=0.6cm, inner sep=0pt, outer sep=0pt, font=\sffamily\footnotesize},
      circle2/.style={draw=blue!70, thick, circle, minimum size=0.6cm, inner sep=0pt, outer sep=0pt, font=\sffamily\footnotesize},
      arrow1/.style={-Stealth, line width = 0.6pt},
      arrow2/.style={-, line width = 0.5pt},
      dashbox/.style={draw=cyan!60!black, rounded corners, dashed, thick, inner sep=3pt, fill=none},
      rect1/.style={draw=cyan!60!black, rounded corners, thick, font=\sffamily\footnotesize},
      rect2/.style={draw=orange!70!gray, rounded corners, thick, font=\sffamily\footnotesize},
      rect3/.style={draw=blue!70!gray, rounded corners, thick, font=\sffamily\footnotesize},
      scale=0.98, transform shape]

  \node[font=\sffamily\footnotesize] (in1) {$\widetilde{\mathbf{W}}=\mathcal{T}_1(\mathbf{w},\mathbf{z},t,\mathbf{x};\mathbf{\theta}_1;p)$};
  \node[fit=(in1), dashbox] (boxn1) {};
  \node[right = 3cm of in1, font=\sffamily\footnotesize] (in3) {$\phi=\mathcal{T}_3(\mathbf{z},t,\mathbf{x};\mathbf{\theta}_3;p)$};
  \node[fit=(in3), dashbox] (boxn3) {};
  \node[below = 3.5cm of in1, font=\sffamily\footnotesize] (in2) {$g=\mathcal{T}_2(\mathbf{z},t,\mathbf{x},\mathbf{v};\mathbf{\theta}_2;p)$};
  \node[fit=(in2), dashbox] (boxn2) {};

  \node[rect3, below=1.1cm of in1, xshift=-1.5cm] (F) {$\overline{\mathbf{F}}=\widetilde{\mathbf{F}}+\delta \mathbf{F}$};

  \node[rect3, below=.8cm of in1, xshift=2.55cm] (U) {$\mathbf{U}=\widetilde{\mathbf{U}}+\delta \mathbf{U}$};

  \node[rect3, above=1.1cm of in2, xshift=.0cm] (f) {$f=\widetilde{\mathcal{M}}+g$};

  \node[right=0.4cm of in2, font=\sffamily\footnotesize] (l3) {$\mathcal{L}_{\text{3}}:\, \frac{\varepsilon}{1+\varepsilon}\left({\partial_t}f + v \cdot \nabla_{\bx} f -\nabla_{\bx} \phi \cdot \nabla_{\bv} f \right) - \frac{1}{1+\varepsilon}\nabla_{\bv} \cdot \left(\mathcal{M} \nabla_{\bv} \left( \frac{f}{\mathcal{M}} \right) \right)$};
  \node[below=0.5cm of in3, xshift=-.6cm, font=\sffamily\footnotesize] (l1) {$\mathcal{L}_1:\,  \Delta_{\bx} \phi - (1 - \rho)$};
  \node[below=1.7cm of in3, xshift=-1.15cm, font=\sffamily\footnotesize] (l2) {$\mathcal{L}_2:\,  \partial_t \mathbf{U} + \nabla_{\bx} \cdot \overline{\mathbf{F}} - \mathbf{S}(\mathbf{U},\phi)$};
  \node[right=3cm of l2, font=\sffamily\footnotesize] (le) {$\mathcal{L}^{\varepsilon}$};
  \node[right=2.58cm of in3, font=\sffamily\footnotesize] (theta) {$\mathbf{\theta}$};
  \node[left=0.cm of le, yshift=1.8cm, font=\sffamily\footnotesize, rotate=90] (Min) {Minimize};

  \node[fit=(l1), rect2] (boxl1) {};
  \node[fit=(l2), rect2] (boxl2) {};
  \node[fit=(l3), rect2] (boxl3) {};
  \node[fit=(le), rect2] (boxle) {};
  \node[fit=(theta), rect2, draw=green!30!gray] (boxtheta) {};

  \draw[arrow1, cyan!40!gray] (boxn1.south) -- (F.north);
  \draw[arrow1, cyan!40!gray] (boxn1.south) -- (U.north);
  \draw[arrow1, cyan!40!gray] (boxn1.south) -- (f.north);

  \draw[arrow1, blue!30!gray] (boxn2.north) -- (F.south);
  \draw[arrow1, blue!30!gray] (boxn2.north) -- (U.south);
  \draw[arrow1, blue!30!gray] (boxn2.north) -- (f.south);

  \draw[arrow1, draw=green!40!gray] (f.east) -- ++(1.18,0) -- ++(0,-0.56) -| (boxl3.north);
  \draw[draw=green!40!gray, line width = 0.6pt] (boxn3.east) -- ++(0.2,0) |- ($(f.east)+(1.2,-0.56)$);
  
  \draw[arrow1, draw=yellow!70!gray] (F.east) -- ($(F)+(3.0,0)$) |- (boxl2.west);
  \draw[draw=yellow!70!gray, line width = 0.6pt] ($(U.south)+(0.02,0)$) -- ++(0,-0.6);
  \draw[draw=yellow!70!gray, line width = 0.6pt] (boxn3.south) -- ($(boxn3)+(0,-0.5)$)-| ($(F)+(3.0,0)$);

  \draw[arrow1, draw=red!50!gray] (U.east) -- ($(U)+(1,0)$) |- (boxl1.west);
  \draw[draw=red!50!gray, line width = 0.6pt] (boxn3.west) -- ($(boxn3.west)+(-1.02,0)$) -- ($(U)+(1,0)$);

  \draw[arrow1, draw=brown!40!gray] (boxl2.east) -- (boxle.west);
  \draw[arrow1, draw=brown!40!gray] (boxl1.east) --++ (0.6, 0) |- (boxle.west);
  \draw[arrow1, draw=brown!40!gray] ($(boxl3.north)+(2,0)$) |- (boxle.west);

  \draw[arrow1, draw=brown!40!gray] (boxle.north) -- (boxtheta.south);

  \end{tikzpicture}

  \caption{UQ-SPINN for Vlasov--Poisson--Fokker--Planck equation: the macro tensor network ${\mathcal T}_1$ builds the anisotropic Maxwellian background, while the tensor network ${\mathcal T}_2$ learns only the kinetic correction. A third tensor network ${\mathcal T}_3$ solves the Poisson problem.}
   \label{F3}
\end{figure}
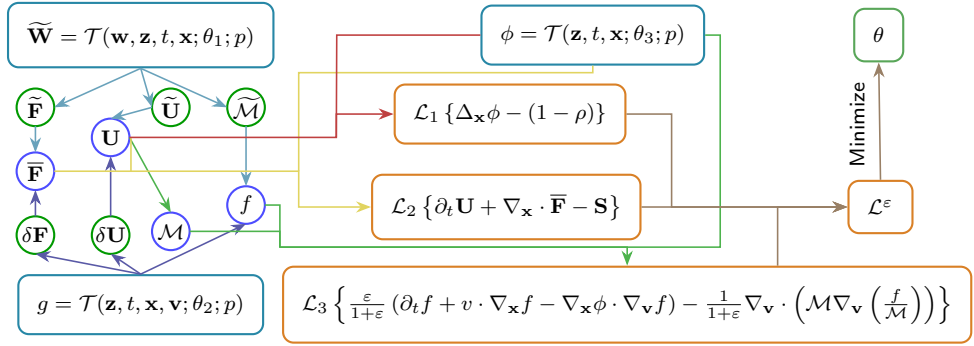

We decompose the distribution function into an anisotropic Maxwellian background $\widetilde{\mathcal{M}}$ and a perturbation $g$.
The neural network is tasked with approximating the computationally inexpensive background $\widetilde{\mathcal{M}}$ to obtain a close surrogate of $f$, while the corrective term $g$ refines this approximation and yields a substantially more accurate solution.
Specifically, we write
\begin{equation}
\label{decom}
  f = \widetilde{\mathcal{M}} + \alpha \; g, \quad {\rm with} \quad \widetilde{\mathcal{M}} = \frac{\rho}{(2\pi)^{d_v/2}\sqrt{{\rm det} \boldsymbol{\Theta}}}  \exp \left( -\frac{1}{2}(\bv - \mathbf{u})^{\top} \boldsymbol{\Theta}^{-1} (\bv - \mathbf{u})\right)
\end{equation}
where $\alpha$ is a tunable constant to facilitate optimization and we set $\alpha = 1$ in our test.
In the anisotropic Maxwellian background $\widetilde{\mathcal{M}}$, we set
\[
\boldsymbol{\Theta} \;=\; \mathbf{R}(\boldsymbol{s})\,\mathbf{T}\,\mathbf{R}(\boldsymbol{s})^{\!\top},
\qquad 
\mathbf{T} \;=\; \operatorname{diag}\!\big(T_1,\dots,T_{d_v}\big)\in\mathbb{R}^{d_v\times d_v},
\]
where \(\mathbf{R}(\boldsymbol{s})\in\mathbb{R}^{d_v\times d_v}\) is an orthogonal matrix and the diagonal entries of \(\mathbf{T}\) are temperatures.
By analogy with the weights \(\mathbf{w}\), one could employ an auxiliary neural network to generate the parameters \(\boldsymbol{s}\), and hence \(\mathbf{R}(\boldsymbol{s})\), thereby enhancing the expressiveness of model at the cost of additional computation. 
In this work, we simply set \(\mathbf{R}(\boldsymbol{s})=\mathbf{I}\), which we find sufficient and performs better than the decomposition in~\cite{oh2025separable}.
Following~\cite{oh2025separable}, to improve the accuracy of velocity–moment integrals, we multiply the perturbation $g$ by a damping factor $\exp (-\bv^2/2)$ that enforces exponential decay near the boundary of the velocity domain. 
The update procedure of the UQ-SPINN is as follows:
\begin{itemize}
  \item \textbf{Step 1.} Given rank $p$, evaluate the three tensor neural networks to obtain $\widetilde{\mathbf{W}}$, $g$, and $\phi$.

  \item \textbf{Step 2.} Given $\widetilde{\mathbf{W}} = (\widetilde{\rho},\widetilde{\mathbf{u}},\widetilde{\mathbf{T}})^{\top}$ and \eqref{decom}, the background Maxwellian $\widetilde{\mathcal{M}}$ is determined.
  By multiplying $\widetilde{\mathcal{M}}$ by the test function $\psi(\bv)$ and $\bv\psi(\bv)$, respectively, and integrating over velocity space, we can obtain
  \begin{equation*}
  \left\{
  \begin{aligned}
    &\widetilde{\mathbf{U}} = \left(\widetilde{\rho},\widetilde{\rho}\widetilde{\mathbf{u}},\widetilde{\mathcal{E}}\right)^{\top}, \\
    &\widetilde{\mathbf{F}} = \left(\widetilde{\rho}\widetilde{\mathbf{u}},\widetilde{\rho}\widetilde{\mathbf{u}}\otimes\widetilde{\mathbf{u}} + \widetilde{\mathbf{P}},\widetilde{\mathcal{E}}\widetilde{\mathbf{u}} + \widetilde{\mathbf{P}}\widetilde{\mathbf{u}}\right)^{\top}, \\
    &\widetilde{\mathcal{E}} = \frac12 \widetilde{\rho} \left(|\widetilde{\mathbf{u}}|^2 + |\widetilde{\mathbf{T}}|\right), \quad \widetilde{\mathbf{P}} = \widetilde{\rho} \;{\rm diag} \widetilde{\mathbf{T}}.
    \end{aligned}
  \right.
  \end{equation*}
  Using the tensor-product structure (linearity) of the integrals, we deduce
  \begin{equation*}
  \left\{
  \begin{aligned}
    &\delta \mathbf{U} = \int_{\mathbb{R}^{d_v}} g \psi(\bv) d\bv = \sum_{j=1}^p \xi_{j,\theta_2}^{\mathbf{z}}\xi_{j,\theta_2}^{t}\xi_{j,\theta_2}^{\mathbf{x}} \int_{\mathbb{R}^{d_v}} \xi_{j,\theta_2}^{\mathbf{v}} \psi(\bv)d\bv, \\
    & \delta \mathbf{F} = \int_{\mathbb{R}^{d_v}} \bv g \psi(\bv) d\bv = \sum_{j=1}^p \xi_{j,\theta_2}^{\mathbf{z}}\xi_{j,\theta_2}^{t}\xi_{j,\theta_2}^{\mathbf{x}} \int_{\mathbb{R}^{d_v}} \bv \xi_{j,\theta_2}^{\mathbf{v}} \psi(\bv)d\bv,\\
    & \xi_{j,\theta_2}^{\mathbf{z}} = \prod_{i = 1}^{d_z} \xi_{j,\theta_2}^{z_i}, \; \xi_{j,\theta_2}^{\mathbf{x}} = \prod_{i = 1}^{d_x} \xi_{j,\theta_2}^{x_i}, \; \xi_{j,\theta_2}^{\mathbf{v}} = \prod_{i = 1}^{d_v} \xi_{j,\theta_2}^{v_i}.
  \end{aligned}
  \right.
  \end{equation*}

  \item \textbf{Step 3.} Compute $f$, $\mathbf{U}$, $\overline{\mathbf{F}}$, and $\mathcal{M}$ as follows
  \begin{equation*}
  \left\{
  \begin{aligned}
    &\mathbf{U} = \int_{\mathbb{R}^{d_v}} f \psi(\bv) d\bv = \int_{\mathbb{R}^{d_v}} (\widetilde{\mathcal{M}} + \alpha \; g) \psi(\bv) d\bv = \widetilde{\mathbf{U}} + \alpha \; \delta \mathbf{U},\\
    &\overline{\mathbf{F}} = \int_{\mathbb{R}^{d_v}} \bv f \psi(\bv) d\bv = \int_{\mathbb{R}^{d_v}} \bv (\widetilde{\mathcal{M}} + \alpha \; g) \psi(\bv) d\bv = \widetilde{\mathbf{F}} + \alpha \; \delta \mathbf{F}, \\
    & {\mathcal{M}} = {\mathcal{M}}({\mathbf{U}};\mathbf{v}).
  \end{aligned}
  \right.
  \end{equation*}

  \item \textbf{Step 4.} Evaluate the data loss and PDE residual terms ($\mathcal{L}_1$, $\mathcal{L}_2$, and $\mathcal{L}_3$), and minimize their sum (the total loss $\mathcal{L}^{\varepsilon}$), to obtain the optimal parameter $\theta$.
\end{itemize}

To prevent numerical blow-up, we scale the PDE in $\mathcal{L}_3$ by $\varepsilon / (1+\varepsilon)$.
This bounded factor ensures the formulation remains well behaved in the limits $\varepsilon \to 0$ and $\varepsilon \to \infty$.
In the limit $\varepsilon \to 0$, the contribution of $\mathcal{L}_3$ reduces to $\mathcal{L}_3 \left\{-\nabla_{\bv} \cdot \left(\mathcal{M} \nabla_{\bv} \left( {f}/{\mathcal{M}} \right) \right)\right\}$, and minimizing this term enforces $f \equiv \mathcal{M}$.
Consequently,
\begin{equation*}
  \overline{\mathbf{F}} = \int_{\mathbb{R}_{d_v}}\bv f\psi(\bv) d\bv= \int_{\mathbb{R}_{d_v}}\bv \mathcal{M} \psi(\bv) d\bv = \mathbf{F}(\mathbf{U}),
\end{equation*}
and the PDE in the UQ-SPINN reduces to
\begin{equation*}
\left\{
\begin{aligned}
  &\partial_t \mathbf{U} + \nabla_{\bx} \cdot \mathbf{F} (\mathbf{U}) = \mathbf{S}(\mathbf{U}, \phi),\\
  &\Delta_\bx \phi = 1 - \rho,
\end{aligned}
\right.
\end{equation*}
which shows that the method satisfies the AP property.
Following~\cite{oh2025separable}, we can also have the following approximation theorem
\begin{lemma}[Approximation on $\Omega\times\mathbb{R}^{d_v}$]
Let $\Omega\subset\mathbb{R}^{d_x}$ be a compact set.
Consider a particle density function $f:\Omega\times\mathbb{R}^{d_v}\to\mathbb{R}$ at a fixed time, and define the velocity weight $l(\mathbf{v})=1+|\mathbf{v}|^{2}$.
Assume that $f\in L^{2}(\Omega\times\mathbb{R}^{d_v})$ and
\[
\int_{\mathbb{R}^{d_v}} f(\cdot,\mathbf{v})\,l(\mathbf{v})\,d\mathbf{v}\ \in\ L^{2}(\Omega).
\]
Then, for every $\epsilon>0$, there exists an UQ-SPINN $f_\theta$ such that
\[
\|f-f_\theta\|_{L^{2}(\Omega\times\mathbb{R}^{d_v})} < \epsilon
\quad\text{and}\quad
\left\| \int_{\mathbb{R}^{d_v}} \big(f-f_\theta\big)\,l(\mathbf{v})\,d\mathbf{v} \right\|_{L^{2}(\Omega)} < \epsilon.
\]
\end{lemma}

\section{Variance Reduction Monte Carlo}
\label{vrmc}
For clarity, we denote the lower-fidelity solvers, ordered by decreasing fidelity relative to the high-fidelity $Q^\mathcal{H}$ solver, by \(Q^{\mathcal{L}_i}\) for \(i=1,\ldots,I\).
We note that the objects $Q := \mathcal Q(f)$ entering the control-variate estimator are quantities of interest associated with the kinetic solution. 
These quantities may be the velocity distribution function itself, or macroscopic observables obtained from it, such as moments.
The variance reduction Monte Carlo (VRMC) can be written as:
\begin{equation}
\label{mpestimator}
\left\{
\begin{aligned}
  &\mathbb{E}[\mathcal Q(f(\bx,\bv,t))]
  \;\approx\;
  E_{K}^{\Lambda}[\mathcal Q(f(\bx,\bv,t))]\\
  &E_{K}^{\Lambda}[\mathcal Q(f(\bx,\bv,t))]
  := \left(
  \frac{1}{K}\sum_{k=1}^{K} Q^{\mathcal{H}}_{k}
  - \sum_{i=1}^{I}\lambda_{i}\!\left(
      \frac{1}{K}\sum_{k=1}^{K} Q^{\mathcal{L}_{i}}_{k}
      - \mathbb{E}\!\left[Q^{\mathcal{L}_{i}}\right]\right)\right)(\bx,\bv,t),
\end{aligned}
\right.
\end{equation}
with the following means and optimized control parameters
\begin{equation*}
  \mathbb{E}\!\left[Q^{\mathcal{L}_{i}}\right]
  := \lim_{L\to\infty}\frac{1}{L}\sum_{k=1}^{L} Q^{\mathcal{L}_{i}}_{k},
  \qquad
  \lambda_i\in\mathbb{R}.
\end{equation*}
Define the control-variate random variable
\begin{equation}
\label{fLambda}
  Q^{\Lambda}(\bz,\bx,\bv,t)
  :=
  Q^{\mathcal{H}}(\bz,\bx,\bv,t)
  - \sum_{i=1}^{I}\lambda_{i}\!\left(
      Q ^{\mathcal{L}_{i}}(\bz,\bx,\bv,t)
      - \mathbb{E}\!\left[Q^{\mathcal{L}_{i}}\right](\bx,\bv,t)\right).
\end{equation}
Introduce the notation
\begin{equation}
\label{notation}
\begin{aligned}
  &\Lambda=(\lambda_{1},\ldots,\lambda_{I})^{\top},\qquad
    b=(b_{i})_{i=1}^{I},\qquad
    C=(c_{ij})_{i,j=1}^{I},\\
  &b_{i}=\operatorname{Cov}\!\big(Q^{\mathcal{H}},Q^{\mathcal{L}_{i}}\big),\qquad
    c_{ij}=\operatorname{Cov}\!\big(Q^{\mathcal{L}_{i}},Q^{\mathcal{L}_{j}}\big),
\end{aligned}
\end{equation}
so that \(C\) is the symmetric \(I\times I\) covariance matrix.
From \eqref{fLambda}–\eqref{notation}, the variance of \(Q^{\Lambda}\) is
\begin{equation}
\label{varfL}
  \operatorname{Var}\!\left(Q^{\Lambda}\right)
  =
  \operatorname{Var}\!\left(Q^{\mathcal{H}}\right)
  + \Lambda^{\top} C \Lambda
  - 2\,\Lambda^{\top} b .
\end{equation}

\begin{lemma}
\label{lem:optCV}
If the covariance matrix \(C\) is nonsingular, then
\begin{equation}
\label{sLambda}
  \tilde{\Lambda}=C^{-1}b
\end{equation}
minimizes \(\operatorname{Var}(Q^{\Lambda})\) at \((\bx,\bv,t)\), and the minimal variance is
\begin{equation}
\label{varfLt}
  \operatorname{Var}\!\left(Q^{\tilde{\Lambda}}\right)
  =
  \left(1-\frac{b^{\top}C^{-1}b}{\operatorname{Var}(Q^{\mathcal{H}})}\right)
  \operatorname{Var}(Q^{\mathcal{H}}).
\end{equation}
\end{lemma}

\begin{proof}
Differentiating \eqref{varfL} with respect to \(\lambda_i\) and setting the derivatives to zero yields
\[
  0 = \frac{\partial}{\partial\lambda_i}\operatorname{Var}(Q^{\Lambda})
    = 2(C\Lambda-b)_i,\quad i=1,\ldots,I,
\]
i.e., \(C\Lambda=b\). If \(C\) is nonsingular, the unique solution is \(\tilde{\Lambda}=C^{-1}b\).
Since the Hessian of \(\operatorname{Var}(Q^{\Lambda})\) with respect to \(\Lambda\) is \(2C\), which is positive definite under the assumption, \(\tilde{\Lambda}\) indeed minimizes the variance.
Substituting \(\tilde{\Lambda}\) into \eqref{varfL} gives \eqref{varfLt}.
\end{proof}
\begin{remark}
\label{2lambda}
To illustrate the methodology, we consider the homogeneous setting with two control variates (\(I=2\)):
\(Q^{\mathcal{L}_1}=f^{0}\) (the initial data) and \(Q^{\mathcal{L}_2}=\mathcal{M}\) (the stationary state).
A direct calculation shows that the optimal coefficients \(\tilde{\lambda}_1\) and \(\tilde{\lambda}_2\) are
\begin{equation}
\label{lambda12}
\begin{aligned}
  \tilde{\lambda}_1
  &= \frac{\operatorname{Var}(\mathcal{M})\,\operatorname{Cov}\!\big(f^{\mathcal{H}}, f^{0}\big)
      - \operatorname{Cov}\!\big(f^{0}, \mathcal{M}\big)\,\operatorname{Cov}\!\big(f^{\mathcal{H}}, \mathcal{M}\big)}
      {\Delta},\\
  \tilde{\lambda}_2
  &= \frac{\operatorname{Var}(f^{0})\,\operatorname{Cov}\!\big(f^{\mathcal{H}}, \mathcal{M}\big)
      - \operatorname{Cov}\!\big(f^{0}, \mathcal{M}\big)\,\operatorname{Cov}\!\big(f^{\mathcal{H}}, f^{0}\big)}
      {\Delta},
\end{aligned}
\end{equation}
where \(\Delta=\operatorname{Var}(f^{0})\,\operatorname{Var}(\mathcal{M})-\operatorname{Cov}\!\big(f^{0},\mathcal{M}\big)^{2}\neq 0\).
Using \(K\) samples for both control variates, the optimal estimator is
\begin{equation*}
\begin{aligned}
  E_{K}^{\tilde{\lambda}_1,\tilde{\lambda}_2}[f](\bv,t)
  &= \frac{1}{K}\sum_{k=1}^{K} f^{\mathcal{H}}_{k}(\bv,t)
   - \tilde{\lambda}_1\!\left(\frac{1}{K}\sum_{k=1}^{K} f^{0}_{k}(\bv) - \mathbb{E}[f^{0}](\bv)\right) \\
  &\quad - \tilde{\lambda}_2\!\left(\frac{1}{K}\sum_{k=1}^{K} \mathcal{M}_{k}(\bv) - \mathbb{E}[\mathcal{M}](\bv)\right).
\end{aligned}
\end{equation*}
At \(t=0\), since \(f^{\mathcal{H}}(\bz,\bv,0)=f^{0}(\bz,\bv)\), we immediately obtain \(\tilde{\lambda}_1=1\) and \(\tilde{\lambda}_2=0\).
Hence the estimator reduces to the exact expression \(E_{K}^{1,0}[f](\bv,0)=\mathbb{E}[f^{0}](\bv)\).
As \(t\to\infty\), if \(f^{\mathcal{H}}(\bz,\bv,t)\to \mathcal{M}(\bz,\bv)\), the coefficients in \eqref{lambda12} satisfy
\[
  \lim_{t\to\infty}\tilde{\lambda}_1=0,
  \qquad
  \lim_{t\to\infty}\tilde{\lambda}_2=1,
\]
and the estimator converges to
\[
  \lim_{t\to\infty} E_{K}^{\tilde{\lambda}_1,\tilde{\lambda}_2}[f](\bv,t)
  = E_{K}^{0,1}[f](\bv) = \mathbb{E}[\mathcal{M}](\bv).
\]
\end{remark}
In general, the optimal coefficients are computed pointwise for the chosen
quantity of interest. More precisely, they depend on the evaluation point
and are determined by the local covariance structure between the high-
and low-fidelity samples. For example, when the quantity of interest is
the distribution function, the coefficients can be written as
\(\lambda_i=\lambda_i(\mathbf{x},\mathbf{v},t)\). In the spatially
homogeneous setting considered in Remark~\ref{2lambda}, this reduces to
\(\lambda_i=\lambda_i(\mathbf{v},t)\). For moment quantities, the
velocity dependence is removed after integration, and the corresponding
coefficients depend on \((\mathbf{x},t)\).

\section{Numerical Examples}
\label{ne}
In the following numerical experiments, MC denotes the standard Monte Carlo method, with samples generated by the deterministic Landau approach.

UQ-SPINN(\(\mathcal{F}_I\)) refers to a VRMC scheme with one low-fidelity
control variate, where samples of the quantity of interest
\(\mathcal{F}\) are produced by applying the UQ-SPINN method to equation
\(I\). 
UQ-SPINN(\(\mathcal{F}_I,\mathcal{F}_J\)) denotes the corresponding
multi-control-variate VRMC scheme with two low-fidelity control variates,
where the UQ-SPINN approximations associated with equations \(I\) and
\(J\) are used jointly to reduce the variance of the high-fidelity VPL
estimator. 
Thus, the latter is the two-control-variate case of the general VRMC
formulation in Section~\ref{vrmc}, rather than a separate variance-reduction
method. 

Some symbols and abbreviations used in the numerical examples are summarized in Table~\ref{tab:mark}.
The calibrated model reported in Table~\ref{tab:mark} refers to the
Fokker--Planck-based surrogate model whose collision frequency is
optimized and whose discrepancy from the VPL solution is further corrected
using reference VPL data, as described in Section~\ref{exam0}.

Throughout the numerical experiments \ref{exam1}-\ref{exam4}, the error is measured for a chosen
quantity of interest \(\mathcal{Q}(f)\), which may be either the
distribution function itself or a macroscopic moment. We use the discrete
mean absolute error
\begin{equation}
    e_{\mathcal{Q}}(t)
    :=
    \frac{1}{N}
    \sum_{m=1}^{N}
    \left|
    \mathbb{E}\!\left[\mathcal{Q}(f)(\xi_m,t)\right]
    -
    E_K^{\Lambda}\!\left[\mathcal{Q}(f)(\xi_m,t)\right]
    \right|,
    \label{eq:error_general}
\end{equation}
where \(\{\xi_m\}_{m=1}^{N}\) denotes the corresponding evaluation grid.
For homogeneous tests, \(\xi_m=\bv_m\); for non-homogeneous moment tests,
\(\xi_m=\bx_m\); and for phase-space quantities, \(\xi_m=(\bx_i,\bv_j)\).
Here \(\mathbb{E}[\mathcal{Q}(f)]\) is approximated by a reference
computation using sufficiently many high-fidelity Landau samples, while
\(E_K^\Lambda[\mathcal{Q}(f)]\) is the VRMC estimator.

For the UQ-SPINN component, we use a SIREN network with sinusoidal activations: an input layer (one feature), two hidden layers of width $128$ (overall depth of three), and a linear output layer of rank $p=256$. 
The SIREN frequency parameter is $w_0=10$. 
Optimization uses the Optax Lion optimizer with zero weight decay and a cosine-decay learning-rate schedule, initialized at $1\times10^{-5}$ (i.e., $10^{-4}/w_0$) and annealed to zero over the total number of training iterations. 
Training was performed on an NVIDIA A30 GPU. 
The deterministic solver is implemented in Fortran and executed on an Intel Xeon Gold 6130 CPU at 2.10\,GHz.
\begin{table}[htbp]
\centering
\label{tab:mark}
\caption{\sf Symbols and abbreviations used in the numerical examples.} 
\footnotesize{
\begin{tabular}{c|c c}
\hline\hline

    Category & Notation & Description \\\hline
\multirow{4}{*}{Subscript} & FP & Homogeneous Fokker--Planck \\
& EP & Euler-Poisson \\
& VPFP & Vlasov-Poisson-Fokker--Planck \\
& VPL & Vlasov--Poisson--Landau \\ \hline
Superscript & * & Calibrated Model     \\\hline\hline
\end{tabular}
}
\end{table}

\subsection{Model calibration}
\label{exam0}
As the VPFP model approximates the VPL equation by replacing its collision operator with a simplified surrogate, its accuracy depends strongly on the choice of the collision frequency $\mu$, which controls the relaxation rate toward equilibrium. 
We assess $\mu$ by computing the mean $L_1$- and $L_\infty$-norm errors, over a dataset of distributions $f$, between the calibrated Fokker--Planck operator $\mu\,P(f)$ and the Landau collision operator $Q(f,f)$. 
The results, summarized in the left panel of Fig.~\ref{CollisionError}, show that the error is minimized at $\mu^{-1} \approx 13$.
On the right of Fig.~\ref{CollisionError}, we plot the time evolution of the model error with respect to VPL for three cases: the standard model ($\mu^{-1}=1$), the calibrated model ($\mu^{-1}=13$), and the calibrated model augmented with VPL data. 
Calibration substantially reduces the discrepancy relative to VPL, and incorporating VPL data further decreases the error.

\begin{remark}
The phrase ``augmented with VPL data'' means that, after calibrating the
collision frequency in the Fokker--Planck model, we further include a
data-matching loss based on a small set of reference VPL solutions. More
precisely, if 
\(f_{\mathrm{VPL}}\) denotes the reference VPL solution at selected
training points, the augmented model is trained with the additional loss
\[
    \mathcal{L}_{\mathrm{data}}
    =
    \frac{1}{N_{\mathrm{data}}}
    \sum_{m=1}^{N_{\mathrm{data}}}
    \left\|
    f_{\theta}(\mathbf z_m,\mathbf{x}_m,\mathbf{v}_m,t_m)
    -
    f_{\mathrm{VPL}}(\mathbf z_m,\mathbf{x}_m,\mathbf{v}_m,t_m)
    \right\|^2 .
\]
Thus, the calibrated Fokker--Planck dynamics provide the model constraint,
while the VPL data are used to correct the remaining discrepancy from the
high-fidelity VPL solution. This is the model labeled as
``Calibration (\(\mu\) and data)'' in the right panel of Fig.~\ref{CollisionError}.
\end{remark}

\begin{figure}[tb]
    \begin{center}
        \mbox{
        {\includegraphics[width = 0.45 \textwidth, trim=0 0 0 0,clip]{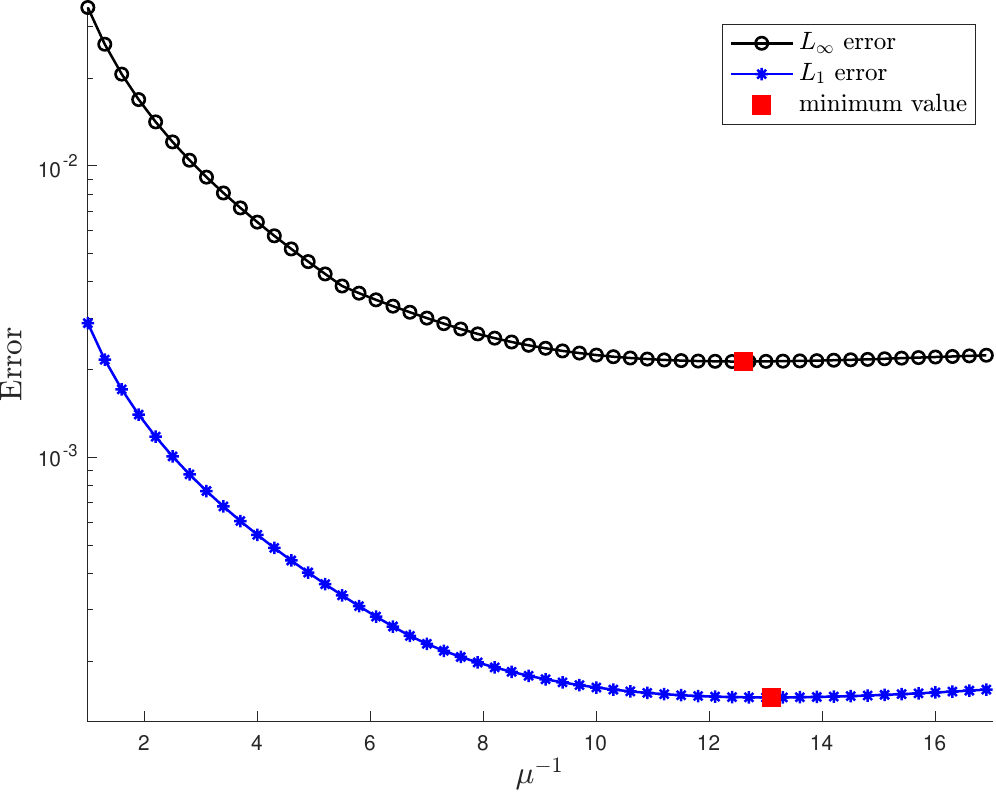}}
        {\includegraphics[width = 0.45 \textwidth, trim=0 0 0 0,clip]{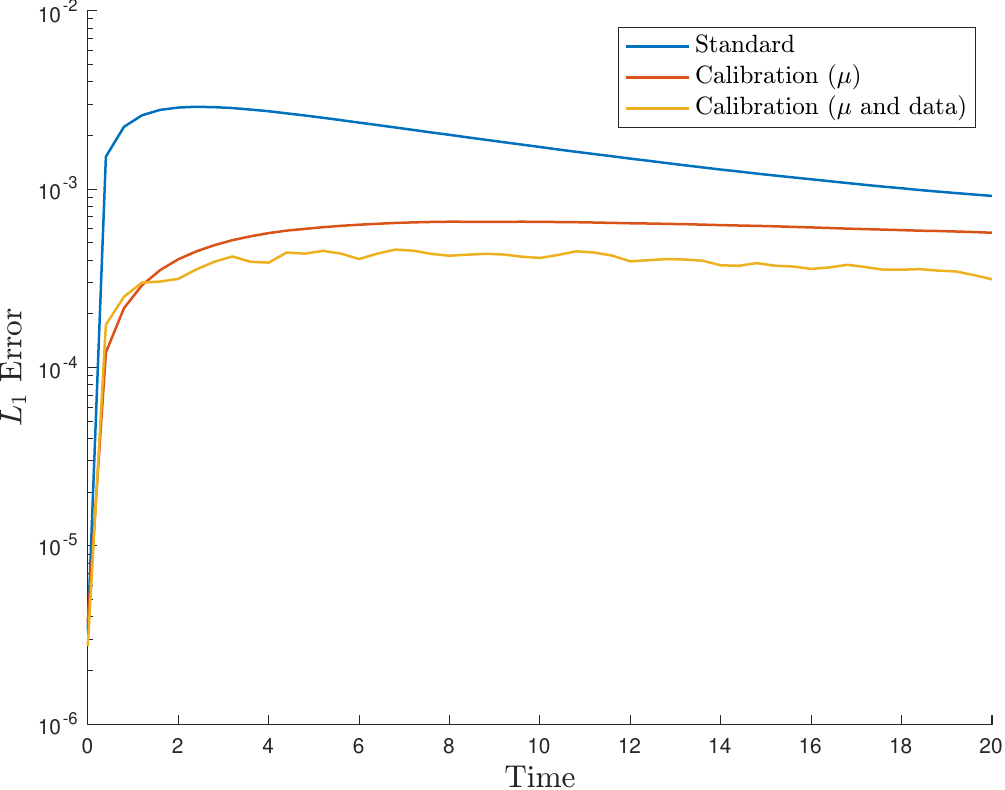}}}
        \caption{\sf Left: the $L_1$ and $L_{\infty}$ errors between Landau operator $Q(f,f)$ and the calibrated Fokker--Planck operator $\mu P(f)$; Right: the model error with respect to VPL for standard model, calibrated model without data, and calibrated model augmented with VPL data.}
        \label{CollisionError}
    \end{center}
\end{figure}

To highlight the benefit of using an anisotropic reference Maxwellian in the decomposition \eqref{decom}, we compare it with the isotropic reference $\widetilde{\mathcal{M}}_{\rm iso} = \mathcal{M}(\widetilde{\mathbf{U}})$ and moment-matched Maxwellian $\widetilde{\mathcal{M}}_{\rm mom} = \mathcal{M}({\mathbf{U}})$.
\(\mathcal{M}(\mathbf U)\), defined in \eqref{maxwellian0}, denotes the isotropic Maxwellian determined by
the macroscopic moments \(\mathbf U=(\rho,\rho\mathbf u,E)^{\top}\).
Here \(\mathbf U\) is computed from the target distribution \(f\), while
\(\widetilde {\mathbf U}\) is the approximate macroscopic vector used for the
isotropic reference Maxwellian.
We apply UQ-SPINN to the homogeneous VPFP equation with the initial condition
{\small
\begin{equation*}
f_0(\mathbf v)
=\frac{3}{16\pi}\!\left(
\exp\!\left[-\frac{(v_x-1.5)^2}{2}-\frac{(v_y-0.5)^2}{4}\right]
+
\exp\!\left[-(v_x+1.5)^2-\frac{(v_y+2.5)^2}{2}\right]
\right).
\end{equation*}}
We emphasize that the background term \(\widetilde{\mathcal{M}}\) is constructed from the \((1+d_x)\)-dimensional macroscopic vector \(\widetilde{\mathbf W}\) via a Maxwellian formula. 
Consequently, it is substantially easier to learn than the \((1+d_x+d_v)\)-dimensional residual \(g\).
When \(\widetilde{\mathcal{M}}\) is accurately approximated, the residual \(g\) only needs to correct the remainder \(f-\widetilde{\mathcal{M}}\), which in turn improves both accuracy and efficiency.
Fig.~\ref{Mg} shows the background \(\widetilde{\mathcal M}\) and the residual \(g\) obtained from the three decompositions.
Relative to the isotropic reference $\widetilde{\mathcal{M}}_{\rm iso}$ and moment-matched case $\widetilde{\mathcal{M}}_{\rm mom}$, the anisotropic Maxwellian \eqref{decom} captures more of the dominant structure of \(f\), yielding a smoother and more learnable residual \(g\).
Therefore, as shown in Fig.~\ref{normandloss}, the resulting error of the anisotropic Maxwellian is approximately half of that produced by the isotropic-Maxwellian baseline.
The moment-matched Maxwellian decomposition yields a more complex \(g\) and, consequently, a training error approximately two orders of magnitude greater than that of the best-performing anisotropic reference decomposition.

\begin{figure}[tb]
    \begin{center}
        \mbox{
        {\includegraphics[width = 0.48 \textwidth, trim=15 0 15 0,clip]{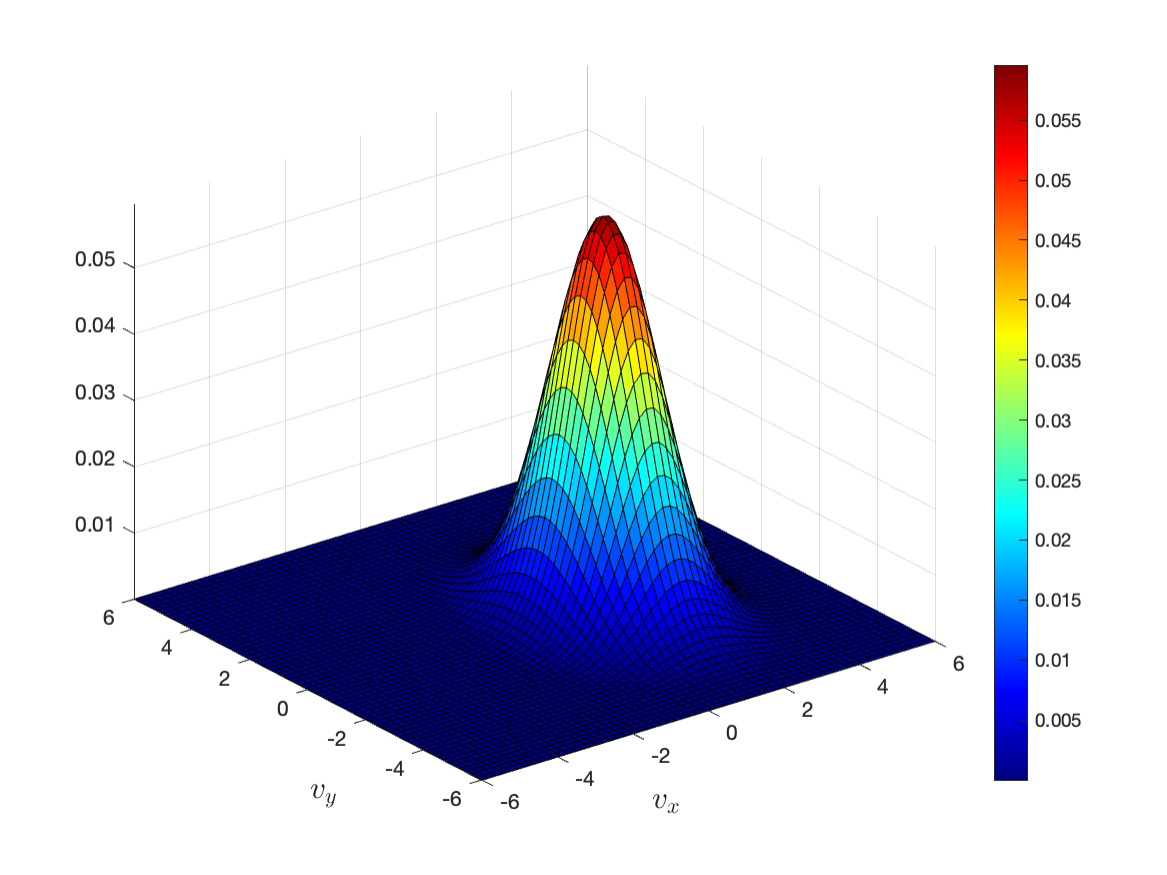}}
        {\includegraphics[width = 0.48 \textwidth, trim=15 0 15 0,clip]{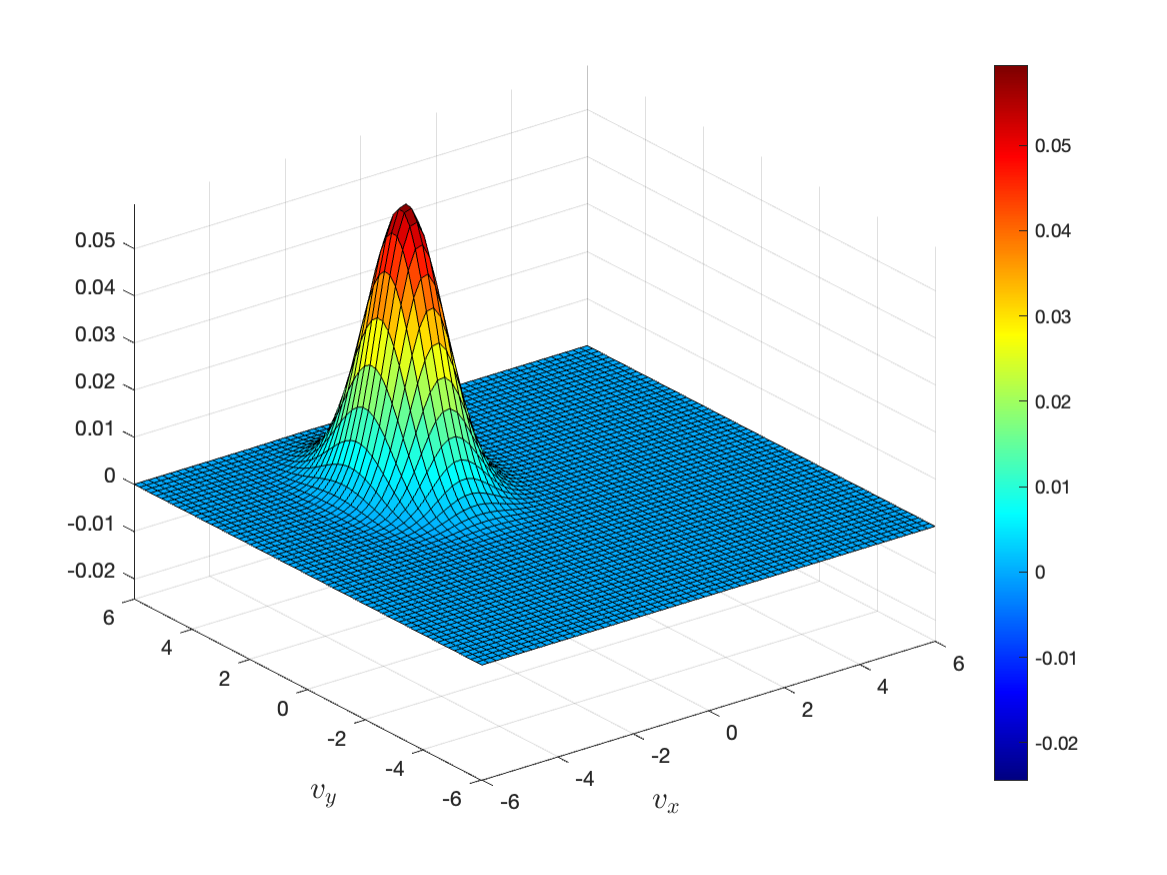}}}
        \mbox{
        {\includegraphics[width = 0.48 \textwidth, trim=15 0 15 0,clip]{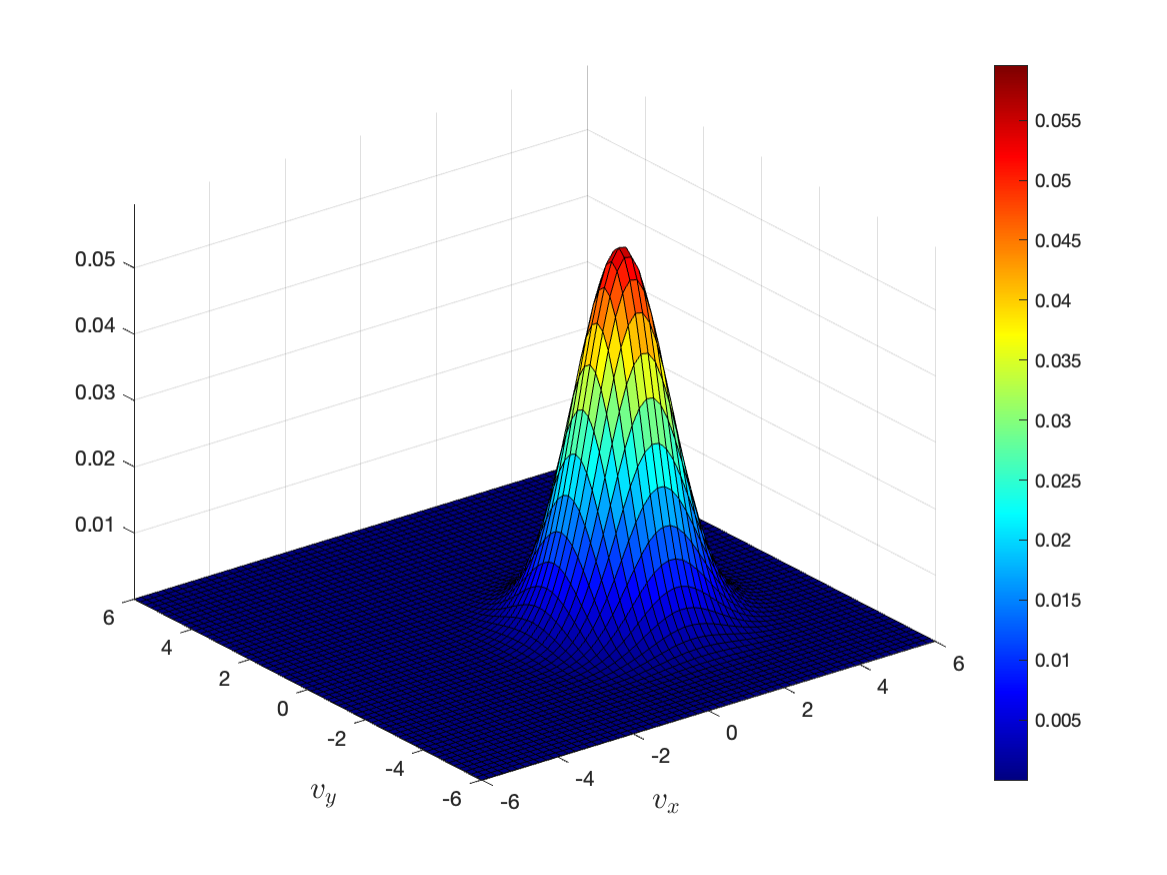}}
        {\includegraphics[width = 0.48 \textwidth, trim=15 0 15 0,clip]{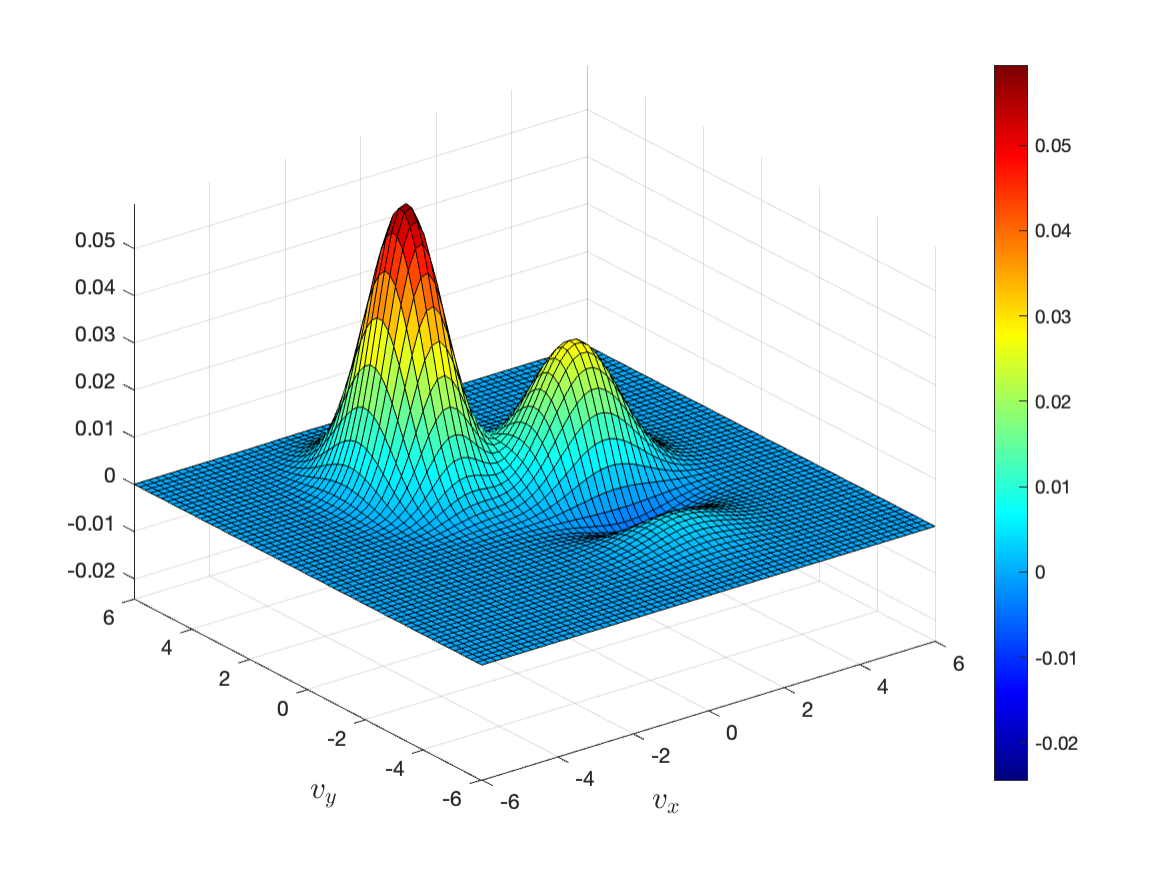}}}
        \mbox{
        {\includegraphics[width = 0.48 \textwidth, trim=15 0 15 0,clip]{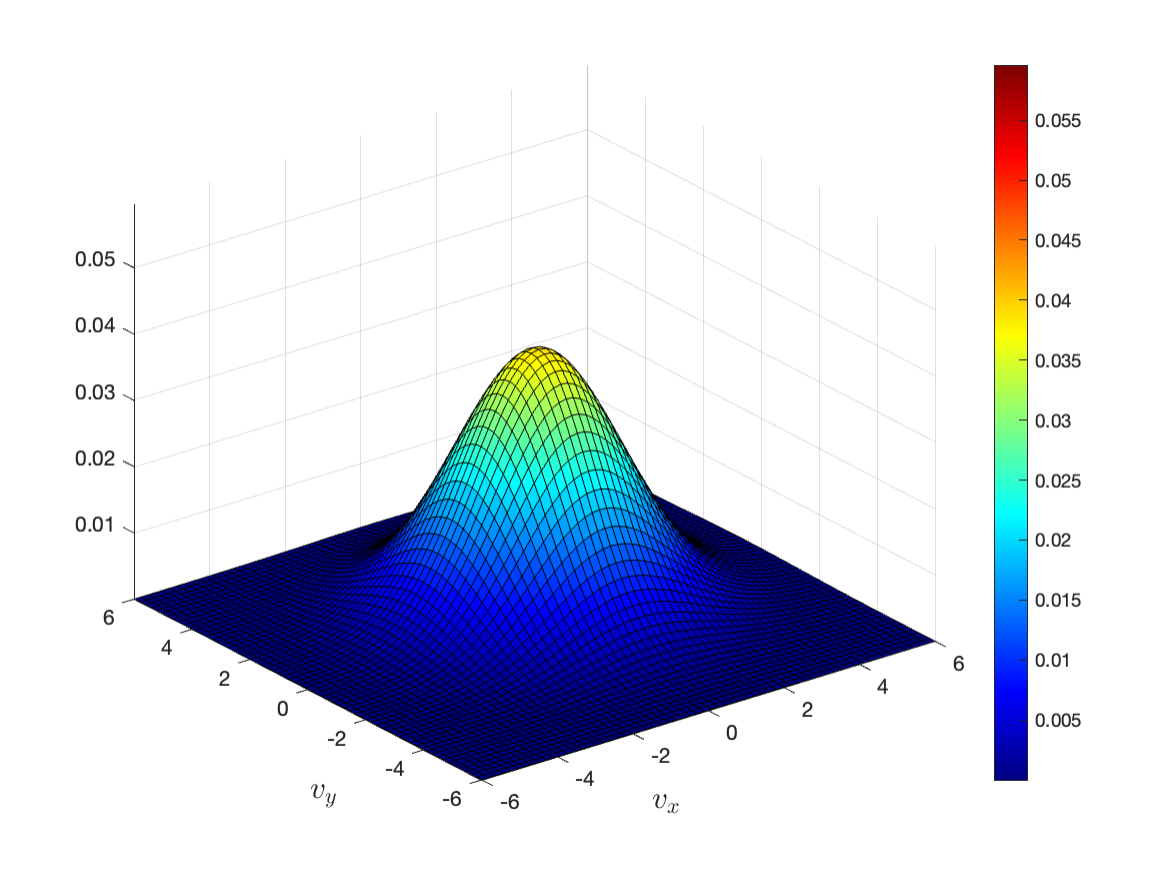}}
        {\includegraphics[width = 0.48 \textwidth, trim=15 0 15 0,clip]{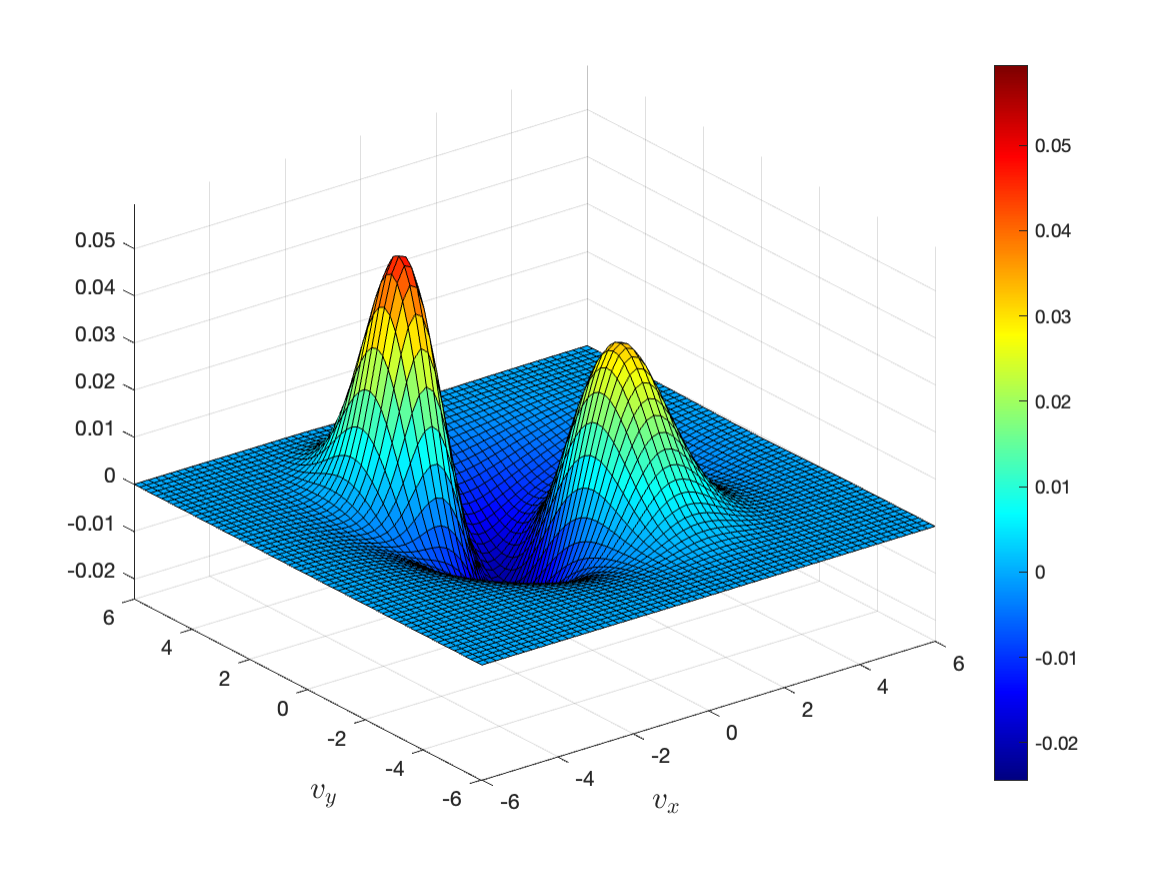}}}
        \caption{\sf Different decomposition. Left: background term \(\widetilde{\mathcal{M}}\); Right: perturbation term \(g\); Top: anisotropic reference Maxwellian; Middle: isotropic reference Maxwellian; Bottom: moment-matched Maxwellian.}
        \label{Mg}
    \end{center}
\end{figure}

\begin{figure}[tb]
    \begin{center}
        \mbox{
        {\includegraphics[width = 0.45 \textwidth, trim=0 0 0 0,clip]{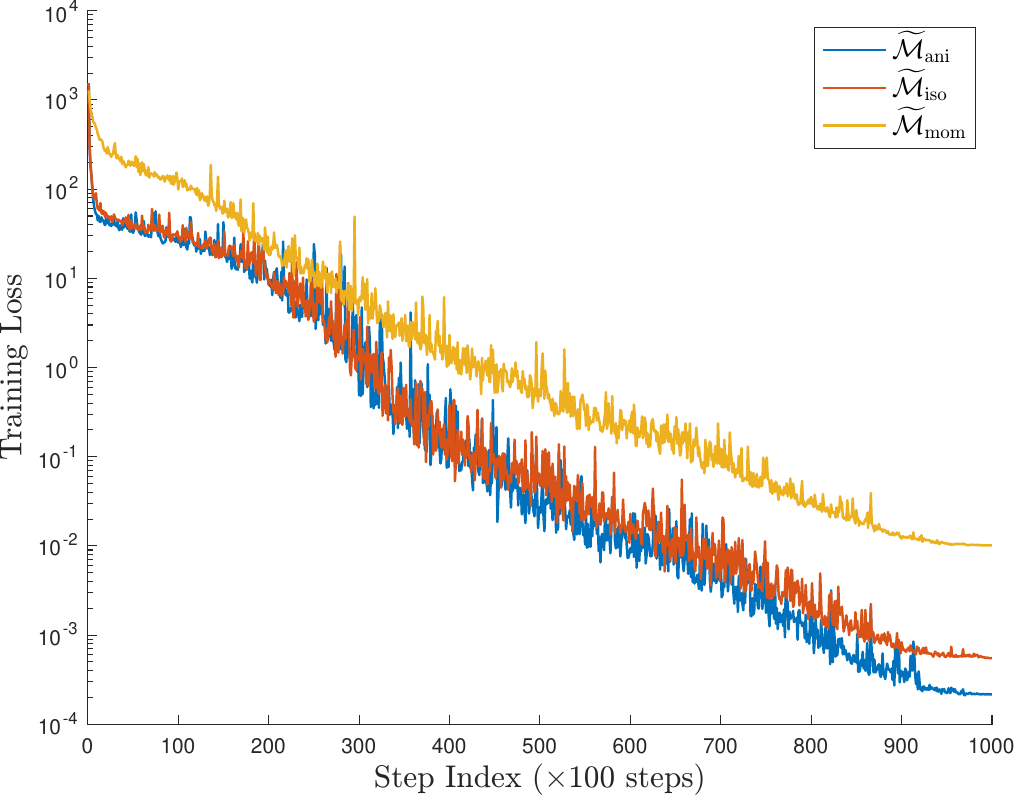}}}
        \caption{\sf Training loss for Example \ref{exam0}.}
        \label{normandloss}
    \end{center}
\end{figure}

\subsection{Two bubbles relaxation}
\label{exam1}
In this example, we study a two-bump problem with uncertainty. 
The initial distribution is
\begin{equation*}
  f_0(\mathbf{z}, \bv) = \frac{\rho_0}{2 \pi} \left( \exp{\left(- \frac{\vert \bv - \mathbf{s} + d \vert^2}{\sigma} \right)} + \exp{\left(- \frac{\vert \bv - \mathbf{s} - d\vert^2}{\sigma} \right)}\right),
\end{equation*}
where $d = 1.5$ and $\mathbf{s}(\mathbf{z}) = z_1(\sin{(2 \pi z_2)}, \cos{(2 \pi z_2)})^{\top}$ denotes the stochastic displacement vector. 
Here, $\mathbf{z} = (z_1, z_2)^{\top} \in [-1, 1] \times [0, 1]$ are random variables: $z_1$ controls the magnitude of displacement, while $z_2$ determines the direction angle. 
The reference density is $\rho_0 = 0.75$, and the parameter $\sigma = 2$ controls the thermal width of the bumps.
The velocity domain is truncated to $\bv \in [-6, 6]^2$. 
For this spatially homogeneous relaxation test, there is no spatial
length scale and hence no physical Knudsen number in the usual sense.
Here \(\varepsilon\) is used as a nondimensional relaxation parameter
multiplying the collision operator, and we set \(\varepsilon=1\) in the
two-bubble benchmark.
The mesh is defined by $N_{\bv} = 64^2$ points for the deterministic solver of Landau equation.
This setup constitutes a five-dimensional problem: two velocity dimensions, two uncertainty dimensions, and one time dimension.

Fig.~\ref{BubblesError} reports the $L_1$ error for short-time ($T=2$) and long-time ($T=20$) behavior across the different methods, using $K=5$ samples for the expected-value estimator and $L=2500$ samples for the control-variate estimator.
As the initial-value problem evolves in time, the bimodal distribution relaxes toward a Maxwellian and the perturbation amplitude decays; consequently, the MC error decreases and eventually plateaus. 
For variance reduction based on the FP collision operator, the FP dynamics approach a Maxwellian much faster than the Landau operator, so the error of UQ initially increases with time and subsequently decreases as the Landau dynamics themselves approach steady state. 
Even so, the FP-based estimator remains more accurate than plain MC throughout. 
Using a calibrated FP operator exhibits a similar temporal error profile but achieves higher accuracy because it more closely matches the Landau operator. 
In both FP-based schemes, optimally chosen weights consistently outperform the fixed weight of one, and a multiple variance-reduction estimator that combines the FP and calibrated-FP operators attains the best accuracy.

\begin{figure}[tb]
    \begin{center}
        \mbox{
        {\includegraphics[width = 0.45 \textwidth, trim=0 0 0 0,clip]{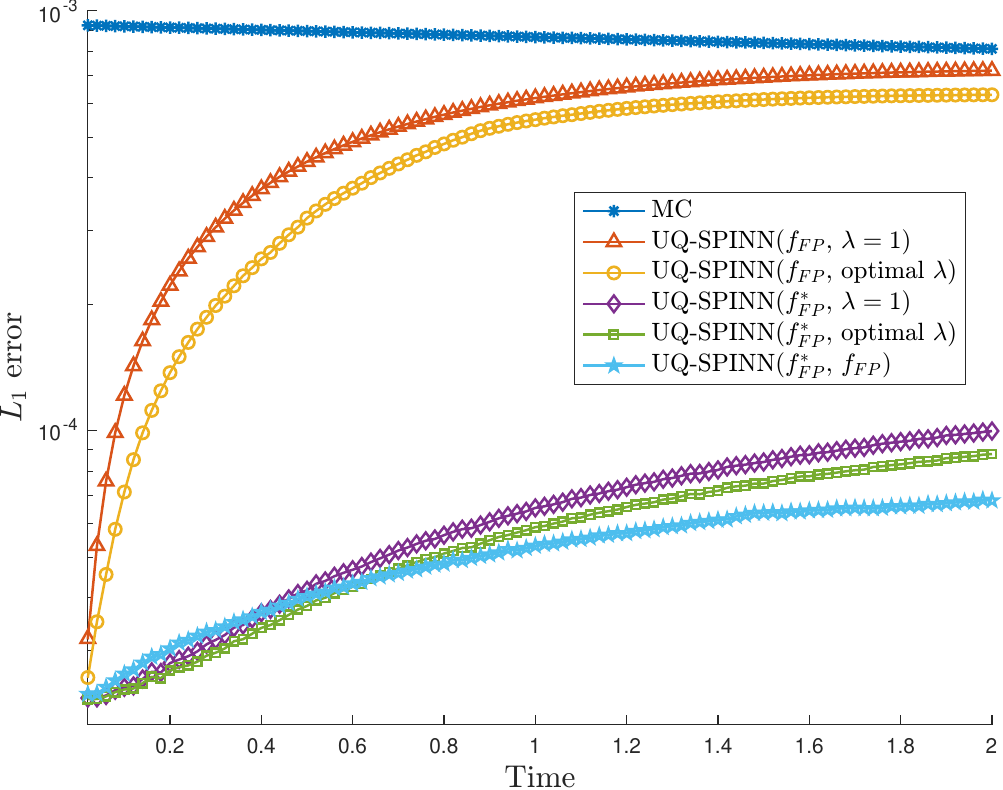}}
        {\includegraphics[width = 0.45 \textwidth, trim=0 0 0 0,clip]{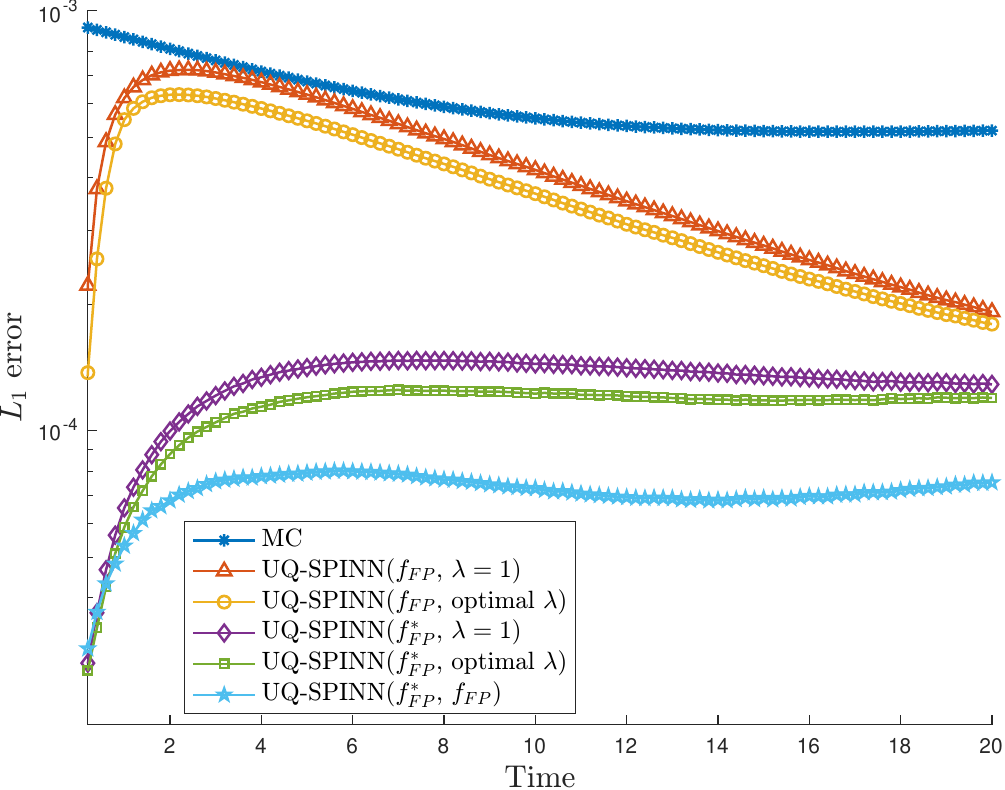}}}
        \caption{\sf $L_1$ error of expectation for Example \ref{exam1}. The sample sizes used to compute the expected value and to construct the control variate are $K = 5$ and $L = 2500$, respectively. Left: short time behavior; Right: long time behavior.}
        \label{BubblesError}
    \end{center}
\end{figure}

To assess the generalization capability of UQ-SPINN, we train the model on a subdomain $\Omega_{z_1}^* \subset \Omega_{z_1} \equiv [-1,1]$ in the $z_1$ direction and generate FP samples over the full domain $\Omega_{z_1}$ to reduce the $\mathbb{E}[f]$ error of Landau. 
The results are reported in Fig.~\ref{BubblesDifferentZ}.
Even when the training region covers only $3/5$ of the evaluation domain, UQ-SPINN substantially improves UQ accuracy relative to MC.
Furthermore, enlarging the training subdomain $\Omega_{z_1}^*$ enhances generalization to the entire domain $\Omega_{z_1}$ and yields further reductions in UQ error.
\begin{figure}[tb]
    \begin{center}
        \mbox{
        {\includegraphics[width = 0.5 \textwidth, trim=0 0 0 0,clip]{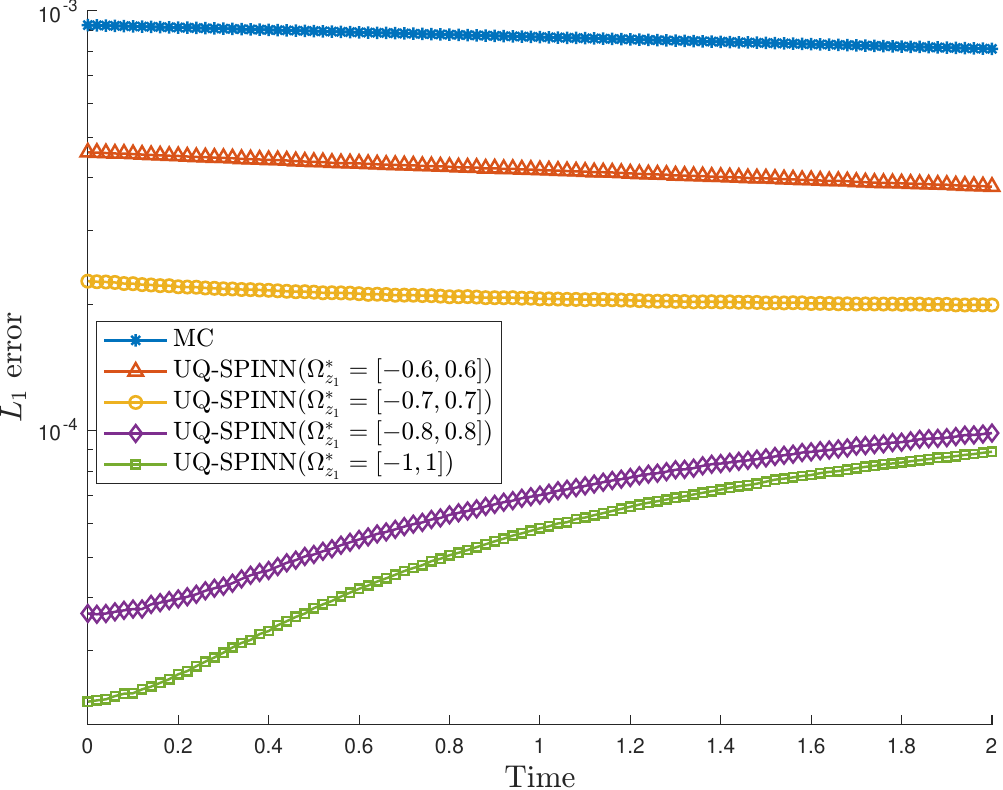}}}
        \caption{\sf $L_1$ error of expectation for Example \ref{exam1} under different training region.}
        \label{BubblesDifferentZ}
    \end{center}
\end{figure}

\subsection{Linear Landau damping}
\label{exam2}
In this test case, we consider the initial condition
\begin{equation*}
  f_0(z, x, \mathbf{v})
  = \frac{1}{2\pi}\!\left(1 + (4 + 2z)\,\alpha \cos(kx)\right)
    \exp\!\left(-\frac{v_x^2 + v_y^2}{2}\right),
\end{equation*}
with $\alpha=0.01$ and $k=0.5$. 
The uncertainty, spatial, and velocity domains are $\Omega_{z}=[0,1]$, $\Omega_x=[0,4\pi]$, and $\Omega_{\mathbf{v}}=[-6,6]^2$, respectively. 
Periodic boundary conditions are imposed on $f$ in space, while in velocity space $f$ satisfies homogeneous Dirichlet conditions on $\partial\Omega_{\mathbf{v}}$.
The Poisson equation is solved with Dirichlet boundary conditions. 
The deterministic solver for VPL uses a mesh with $N_x=65$ spatial points and $N_{\mathbf{v}}=32^2$ velocity points. 
We monitor the base-10 logarithm of the $L_2$ norm of the electric field,
\begin{equation*}
  \zeta(t)
  = \log_{10}\!\left\|\mathbf{E}(t,\cdot)\right\|_{L^2(\Omega_x)}
  = \log_{10}\!\left(\left[\int_{\Omega_x}\lvert \mathbf{E}(t,x)\rvert^2\,\mathrm{d}x\right]^{1/2}\right),
\end{equation*}
and simulate up to the final time $T=10$ for three Knudsen numbers: collisionless $\varepsilon=1e+6$, mild collisions $\varepsilon=1$, and strong collisions $\varepsilon=1e-4$.

Neural network surrogates for PDEs degrade over long horizons as stepwise errors compound and the state drifts off the physical manifold. 
To stabilize the long-time behavior, we train the model on successive time windows.
We compare global training on the full horizon with windowed training that partitions the interval into \([0,4]\), \([4,7]\), and \([7,10]\). 
Under the collisionless setting---where the VPL reference and the VPFP model coincide---the left panel of Fig.~\ref{LLDCompare} shows the time evolution of \(\zeta\): windowed training closely matches the reference, whereas global training performs poorly and fails to capture the pertinent structures. 
For the mild-collision regime, shown in the right panel of Fig.~\ref{LLDCompare}, the VPL and VPFP references exhibit a marked discrepancy; incorporating calibration parameter $\mu$ and data-correction term in UQ-SPINN drives the predictions toward the VPL reference.

We draw $K = 15$ parameter values \(z\sim\mathcal{U}[0,1]\) and compute VPL samples with a deterministic solver. 
Comparing the sample-mean density with the reference density yields the MC error shown in the left panel of Fig.~\ref{LLD}. 
Using these samples as supervision, we train a UQ-SPINN based on the calibrated VPFP system; to obtain better results, we partition the time interval into five equal subintervals, each of length~2. 
As shown in the right panel of Fig.~\ref{LLD}, evaluating the trained model at \(L=20000\) random inputs \(z\) produces \(\mathbb{E}\!\left[\zeta^{*}_{{VPFP}}\right]\) together with its confidence band, across collision regimes, \(\mathbb{E}\!\left[\zeta^{*}_{VPFP}\right]\) closely matches the reference \(\mathbb{E}[\zeta]\). 
We then combine these $L$ calibrated VPFP samples with the $K$ deterministic VPL samples to construct a variance-reduction estimator UQ-SPINN($\rho^*_{VPFP}$), which reduces the MC density error by roughly two orders of magnitude. 
We also train an EP-based UQ-SPINN and generate $L$ samples. 
At \(\varepsilon=1e-4\), its error is comparable to that of the calibrated VPFP-based model, consistent with the fact that both VPFP and VPL converge to EP in this limit. 
For \(\varepsilon\gg1e-4\), the calibrated VPFP model delivers a more stable and pronounced error reduction than EP. 
We note that the density amplitude in the vicinity of a discontinuity is extremely small, and—unless very high-order, low-diffusion numerics are used—these fine features are smoothed out. 
Moreover, the discontinuity locations in EP do not coincide with those of the large-\(\varepsilon\) cases observed for VPFP/VPL (e.g., \(\varepsilon=1\) or \(1e+6\)). 
Consequently, at positions where a discontinuity is present for large \(\varepsilon\) but absent for EP, the EP solution retains a larger amplitude and, because the dependence on \(z\) is nearly linear for both the density and \(\zeta\) in this regime, carries more information content, allowing EP to surpass calibrated VPFP locally. 
Conversely, whenever EP exhibits a discontinuity while the large-\(\varepsilon\) model does not, EP provides less information and its variance-reduction effect weakens; overall, EP is markedly less stable than calibrated VPFP. 
As nonlinearity increases (Example \ref{exam3}), this advantage further diminishes: the \(z\)-dependence becomes more intricate and regime-dependent—see Fig.~\ref{NLD} (right) for the comparison of \(\zeta\) at \(z=0\) and \(z=1\)—and the variance-reduction capability of EP samples degrades.
Finally, the EP-based and calibrated VPFP-based surrogates have already produced enough samples to realize variance reduction. 
Under the accuracy limits of the neural networks, multiple VRMC does not further decrease the variance and may even degrade accuracy because of finite-precision issues in the weights; accordingly, we omit those results. 

Table~\ref{tab:runtime} reports separately the UQ-SPINN one-time offline training cost of the neural-network surrogates and the online inference cost required to generate additional samples for the different test cases; the latter is the relevant quantity in the variance-reduction setting, where tens of thousands of realizations are required.

For this test case, assuming the same \(20000\) samples, the wall-clock
times of the EP-based UQ-SPINN, the calibrated-VPFP-based UQ-SPINN, and
the deterministic solvers for EP, calibrated VPFP, and VPL are summarized
in Table~\ref{tab:wallclock_time}. The comparison shows that the efficiency
improvement of UQ-SPINN for UQ computation, even including the offline
training cost, is substantial with respect to the deterministic baselines.
For the EP model, whose deterministic solver is already relatively cheap,
direct sampling may still be a viable option in moderate-dimensional or
moderate-sample regimes. By contrast, the gain is much more pronounced for
the calibrated VPFP model: direct VPFP sampling remains computationally very
expensive, whereas the trained VPFP-based UQ-SPINN generates additional
realizations at very small marginal cost.
\begin{table}[htbp]
\centering
\caption{Wall-clock time comparison for \(20000\) samples.}
\label{tab:wallclock_time}
\begin{tabular}{lcc}
\hline
Method & Wall-clock time & Relative ratio \\
\hline
EP-based UQ-SPINN
& \(493.132 + 0.009\times 200\)
& \(1\) \\
Calibrated-VPFP-based UQ-SPINN
& \(2312.217 + 0.015\times 200\)
& \(4.68\) \\
Deterministic EP solver
& \(13.828\times 200\)
& \(5.58\) \\
Deterministic calibrated VPFP solver
& \(72633.966\times 200\)
& \(2.94\times 10^4\) \\
Deterministic VPL solver
& \(101151.912\times 200\)
& \(4.09\times 10^4\) \\
\hline
\end{tabular}
\end{table}

\begin{remark}[Surrogate models and computational cost]
In all our experiments, we set the sample size for UQ-SPINN using either EP or the calibrated VPFP model as control variates to \(L = 20000\). Since neural-network inference is extremely fast (Table~\ref{tab:runtime}), one could in principle increase \(L\) to obtain a more accurate estimate of the control-variate mean and further reduce the Monte Carlo error. In practice, however, the achievable accuracy is limited by the surrogate error: for \(L > 20000\), additional samples do not improve the estimated control-variate mean.

Moreover, the training-time ratio for UQ-SPINN with calibrated VPFP versus EP as control variates ranges between $4.7$ to $6.3$, while the corresponding inference-time ratio span between $1.12$ and $1.7$. Under a fixed computational budget, these ratios would suggest allocating more samples to EP than to calibrated VPFP, leading to a comparison between two competing effects: a weakly correlated but highly sampled control variate (EP) versus a strongly correlated but sparsely sampled one (calibrated VPFP).

This issue is particularly relevant when considering direct deterministic solvers. In that setting, the high computational cost of VPFP makes it only a modest low-fidelity model, since it allows only a limited increase in the number of samples compared to direct VPL simulations, whereas EP enables orders-of-magnitude more samples. By contrast, VPFP is an excellent surrogate for training neural networks, as it replaces the Landau integral operator with a differential operator that can be efficiently evaluated within the loss function.

Finally, because inference time is negligible compared to training time for \(L \le 20000\), and increasing \(L\) beyond this value yields no further improvement due to surrogate error saturation, we did not perform the usual cost optimization by varying \(L\) between EP- and VPFP-based surrogates. This highlights a fundamental difference between direct solvers and UQ-SPINN when used as surrogate models.
\end{remark}

\begin{figure}[tb]
    \begin{center}
        \mbox{
        {\includegraphics[width = 0.45 \textwidth, trim=0 0 0 0,clip]{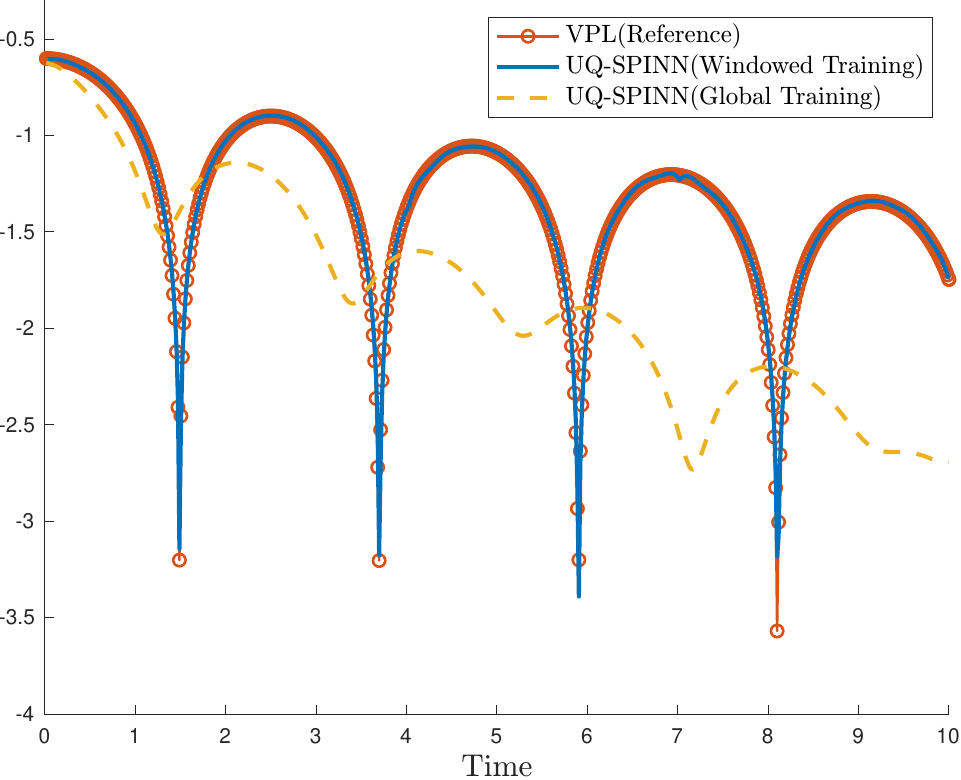}}
        {\includegraphics[width = 0.45 \textwidth, trim=0 0 0 0,clip]{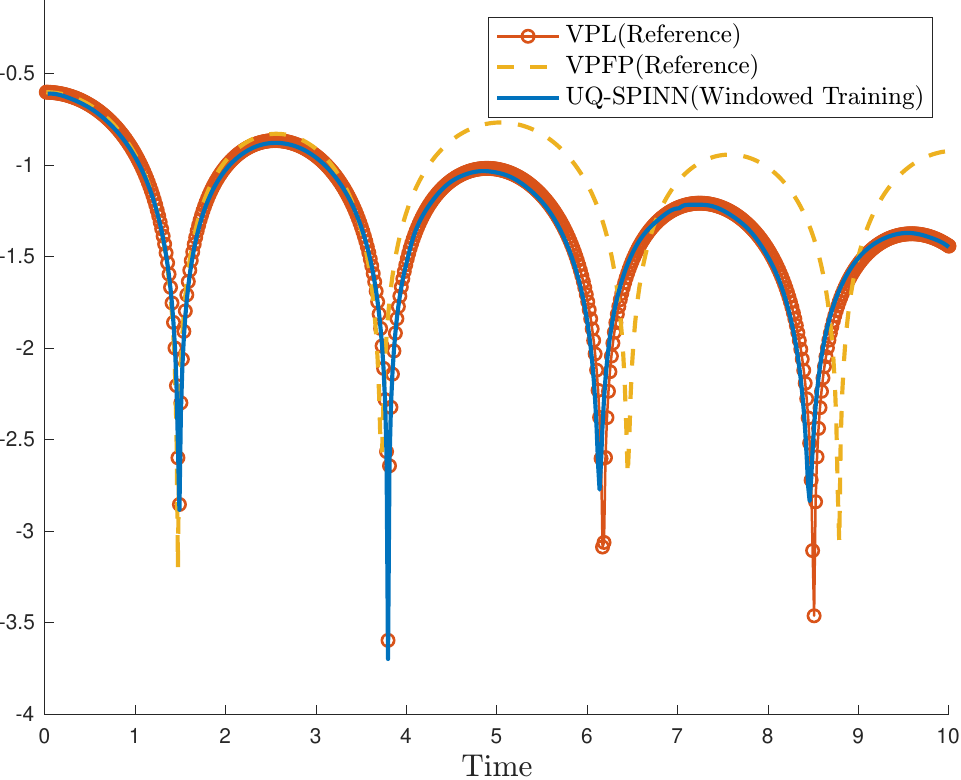}}}
        \caption{\sf Left: global training and windowed training at $\varepsilon = 1e+6$; Right: windowed training at $\varepsilon = 1$.} 
        \label{LLDCompare}
    \end{center}
\end{figure}

\begin{figure}[tb]
    \begin{center}
        \mbox{
        {\includegraphics[width = 0.45 \textwidth, trim=0 0 0 0,clip]{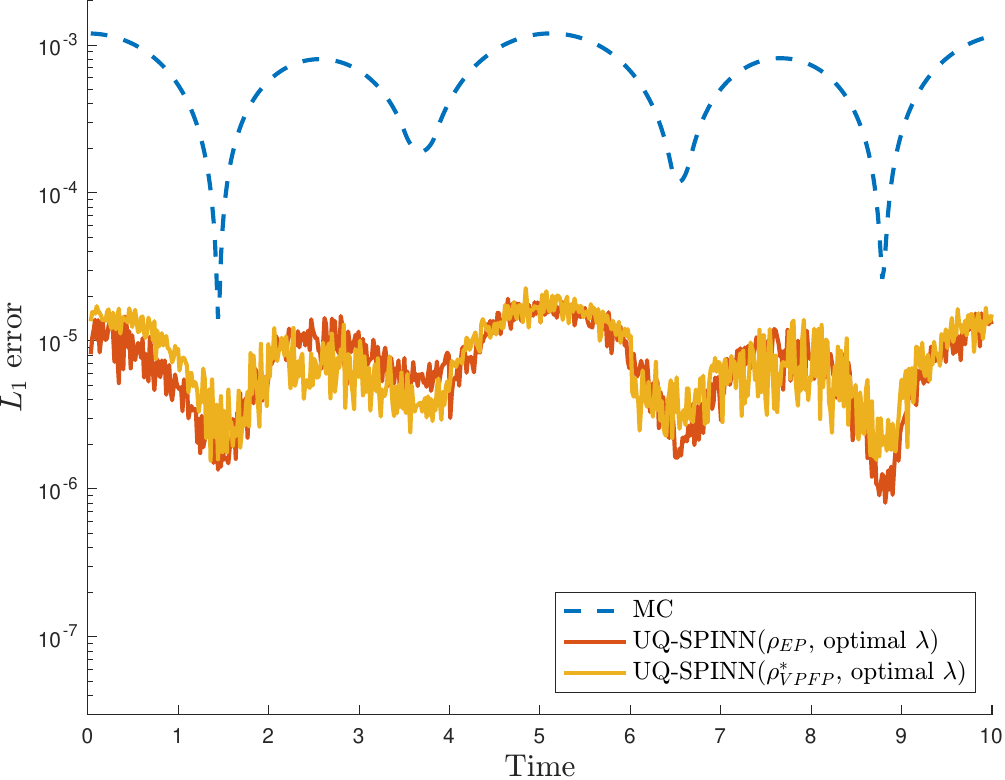}}
        {\includegraphics[width = 0.45 \textwidth, trim=0 0 0 0,clip]{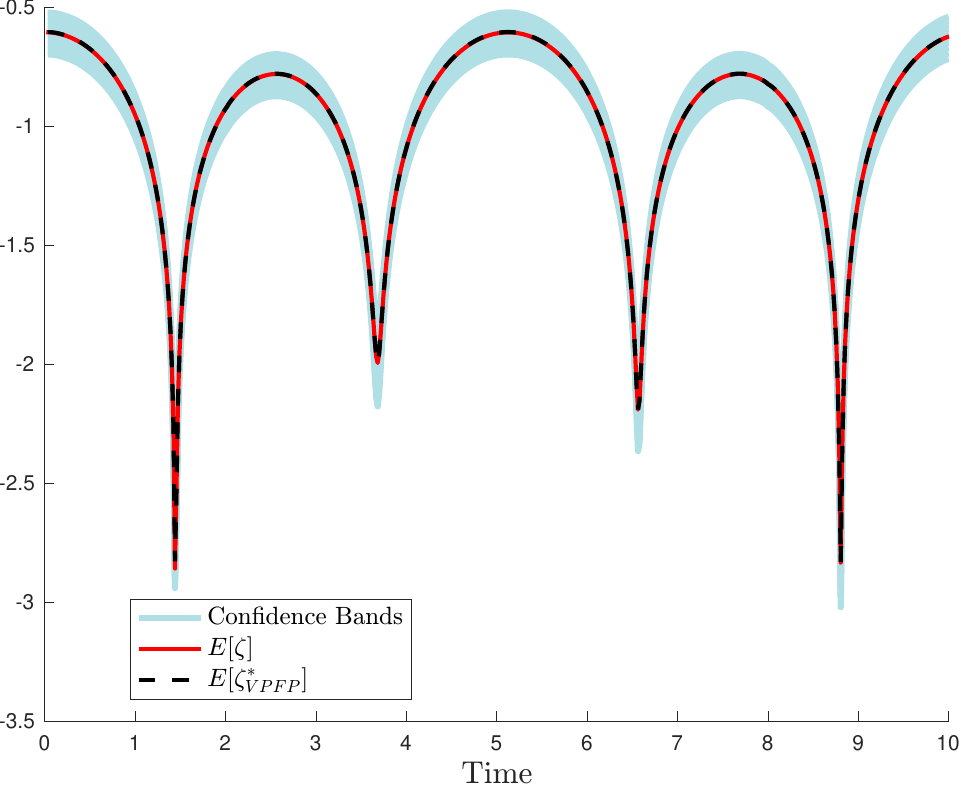}}}
        \mbox{
        {\includegraphics[width = 0.45 \textwidth, trim=0 0 0 0,clip]{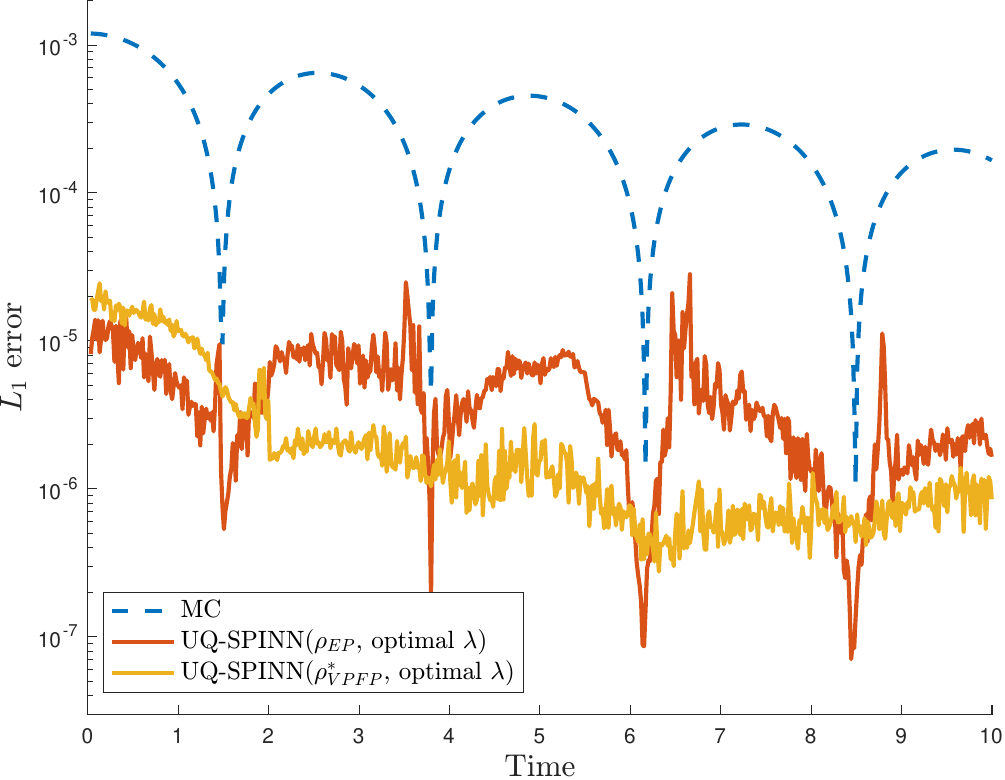}}
        {\includegraphics[width = 0.45 \textwidth, trim=0 0 0 0,clip]{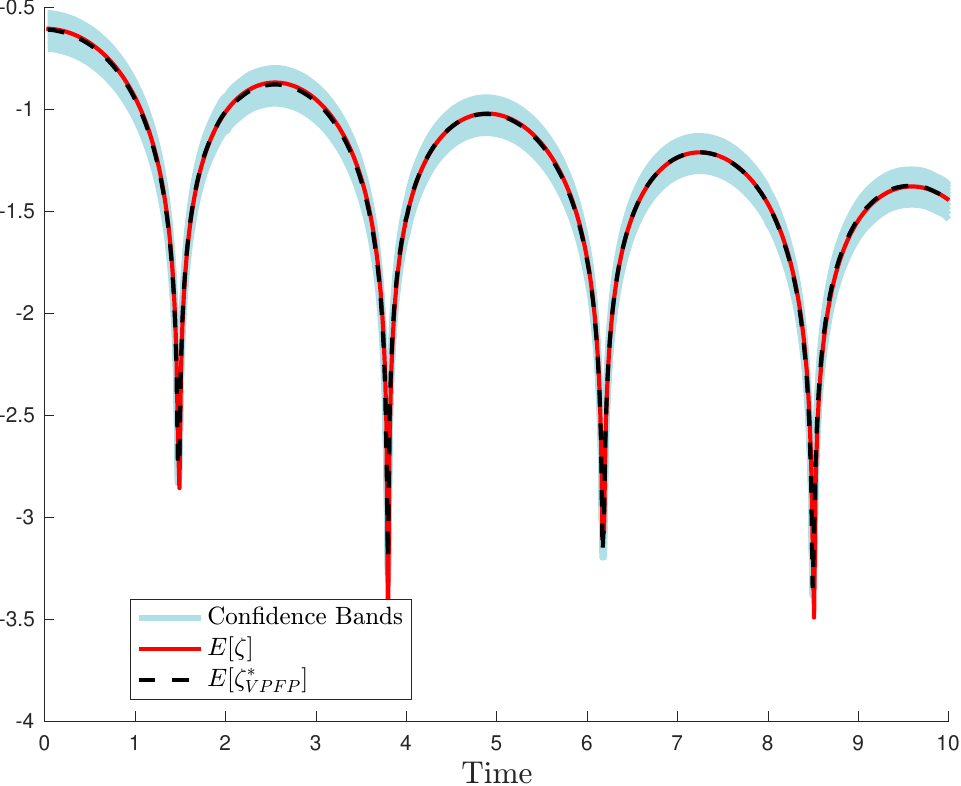}}}
        \mbox{
        {\includegraphics[width = 0.45 \textwidth, trim=0 0 0 0,clip]{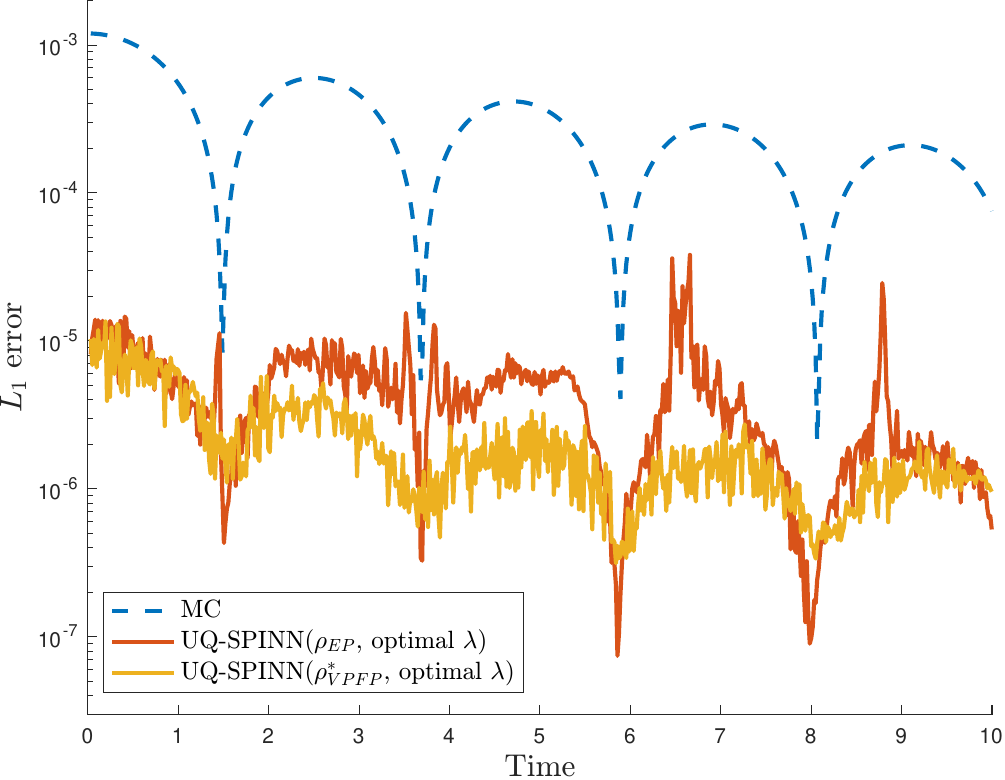}}
        {\includegraphics[width = 0.45 \textwidth, trim=0 0 0 0,clip]{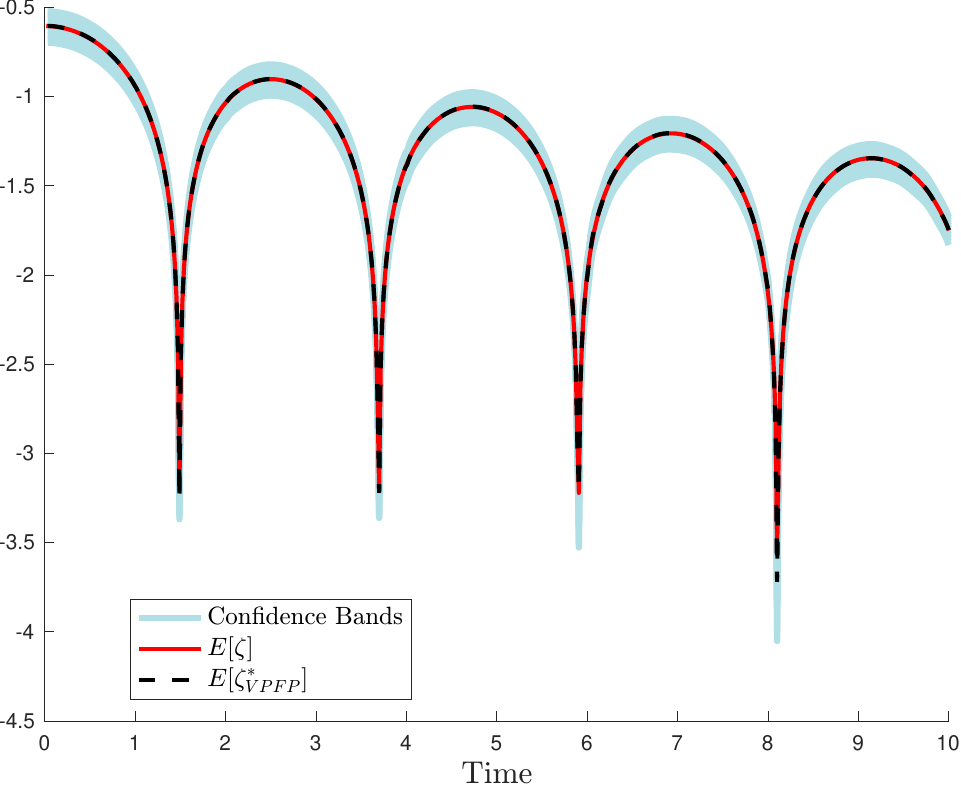}}}
        \caption{\sf Results for Example \ref{exam2}. The sample sizes used to compute the expected value and to construct the control variate are $K = 15$, $L = 20000$ for calibrated VPFP and EP, respectively. Left: $L_1$ error of expectation $\mathbb{E}[\rho]$; Right: expectation $\mathbb{E}[\zeta]$ and confidence bands; Top: $\varepsilon = 1e-4$; Middle: $\varepsilon = 1$; Bottom: $\varepsilon = 1e+6$.} 
        \label{LLD}
    \end{center}
\end{figure}

\begin{table}[htbp]
\centering
\label{tab:runtime}
\caption{\sf Runtime comparison for Example \ref{exam2} -- \ref{exam4}. “Training” is the wall-clock training time for the UQ-SPINN (not applicable to the deterministic baseline). 
“Inference” is the wall-clock time to evaluate 100 samples. } 
\footnotesize{
\begin{tabular}{c|c c c c}
\hline\hline
\multirow{2}{*}{Example} & \multirow{2}{*}{Method} & \multirow{2}{*}{Equation} & {Training} & {Inference} \\
 & & &(s) & (s / 100 samples) \\\hline
\multirow{5}{*}{\ref{exam2}} &UQ-SPINN & EP   & 493.132    & 0.009      \\
                             &UQ-SPINN                  & calibrated VPFP & 2312.217  & 0.015      \\
                             &Deterministic Solver   & EP  & N/A       & 13.828  \\
                             &Deterministic Solver   & calibrated VPFP  & N/A       & 72633.966  \\
                             &Deterministic Solver   & VPL  & N/A       & 101151.912  \\ \hline
\multirow{5}{*}{\ref{exam3}} &UQ-SPINN & EP   & 1735.592   & 0.032      \\
                             &UQ-SPINN                  & calibrated VPFP & 8640.782  & 0.055      \\
                             &Deterministic Solver   & EP  & N/A       & 60.572  \\
                             &Deterministic Solver   & calibrated VPFP  & N/A       & 402232.129  \\
                             &Deterministic Solver   & VPL  & N/A       & 554730.226  \\ \hline
\multirow{5}{*}{\ref{exam4}} &UQ-SPINN & EP   & 980.097    & 0.033     \\
                             &UQ-SPINN                  & calibrated VPFP & 6180.826  & 0.037     \\
                             &Deterministic Solver   & EP  & N/A       & 9.948  \\
                             &Deterministic Solver   & calibrated VPFP  & N/A       & 33165.674  \\
                             &Deterministic Solver   & VPL  & N/A       & 46005.180  \\\hline\hline
\end{tabular}
}
\end{table}

\subsection{Nonlinear Landau damping}
\label{exam3}
In this case, we consider the nonlinear Landau damping with the initial distribution
\begin{equation*}
  f_0(z, x, \mathbf{v})
  = \frac{1}{2\pi}\!\left(1 + (1 + 3z)\,\alpha \cos(kx)\right)
    \exp\!\left(-\frac{v_x^2 + v_y^2}{2}\right),
\end{equation*}
with $\alpha = 0.1$ and $k = 0.5$.
The uncertainty, spatial, and velocity domains are $\Omega_{z}=[0,1]$, $\Omega_x=[0,4\pi]$, and $\Omega_{\mathbf{v}}=[-6,6]^2$, respectively. 
The boundary conditions are the same as those in Example \ref{exam2}.
The deterministic solver for VPL uses a mesh with $N_x=65$ spatial points and $N_{\mathbf{v}}=64^2$ velocity points. 
We evaluate $\zeta(t)$ up to the final time $T=36$ for three Knudsen numbers: collisionless $\varepsilon=1e+6$, mild collisions $\varepsilon=1$, and strong collisions $\varepsilon=1e-4$.
The time interval is divided into 18 equal subintervals, each of length~2. 
We set the sample sizes to $K=15$ (for estimating the expectation and training UQ-SPINN) and $L=20000$ (for constructing the control variates). 

In the left panel of Fig.~\ref{NLD}, we report the errors in the expected density for MC, UQ-SPINN\((\rho_{{EP}})\), and UQ-SPINN\((\rho^{*}_{{VPFP}})\). 
Because calibrated VPFP and VPL both converge to the limiting EP model as \(\varepsilon\to 0\), at \(\varepsilon=1e-4\) the two surrogates achieve comparable UQ-error reductions, with UQ-SPINN\((\rho^{*}_{{VPFP}})\) lowering the MC error by approximately two orders of magnitude. 
As \(\varepsilon\) increases (i.e., the Knudsen number departs from zero), the variance-reduction effect of UQ-SPINN\((\rho_{{EP}})\) degrades for two reasons: (i) EP is only the zero-Knudsen limit of VPL, so their solutions separate as the Knudsen number grows; and (ii) the dependence of \(\zeta\) on \(z\) becomes more intricate and nonlinear than in the linear Landau damping regime—see Fig.~\ref{NLD}(right) for the comparison at \(z=0\) and \(z=1\). 
By contrast, \(\mathbb{E}[\zeta^{*}_{{VPFP}}]\) closely matches the VPL reference \(\mathbb{E}[\zeta]\) across Knudsen numbers—thanks to data/parameter calibration and windowed training—thus providing high-quality control variates and consistently delivering about two orders of magnitude improvement in UQ accuracy. 

A subtlety arises at \(\varepsilon=1\): as time evolves, the density amplitude becomes so small that, unless very high-order, low-diffusion numerical solvers are used, the fine-scale features are smoothed away and the information content is lost. 
In this tiny-amplitude regime, UQ-SPINN\((\rho^{*}_{{VPFP}})\) cannot further improve UQ accuracy, whereas the EP-based surrogate UQ-SPINN\((\rho_{{EP}})\) retains a larger amplitude than the corresponding mild collisions VPFP/VPL solution and still provides modest variance reduction. 
We refer readers to Remark \ref{NLDRemark} for more details.

Finally, Table~\ref{tab:runtime} summarizes training and sample-generation times.
Assuming the same 20000 samples, the ratio of wall-clock times for the EP-based UQ-SPINN, the calibrated-VPFP-based UQ-SPINN, and the deterministic solvers of EP, calibrated VPFP, and VPL is approximately $1 : 4.97 : 6.96: 4.62 \times 10^{4} :6.37 \times 10^{4}$.
Combined with the numerical results above, these timings show that UQ-SPINN greatly increases sample-generation efficiency, thereby substantially reducing variance and enhancing uncertainty quantification accuracy.

\begin{figure}[tb]
    \begin{center}
        \mbox{
        {\includegraphics[width = 0.45 \textwidth, trim=0 0 0 0,clip]{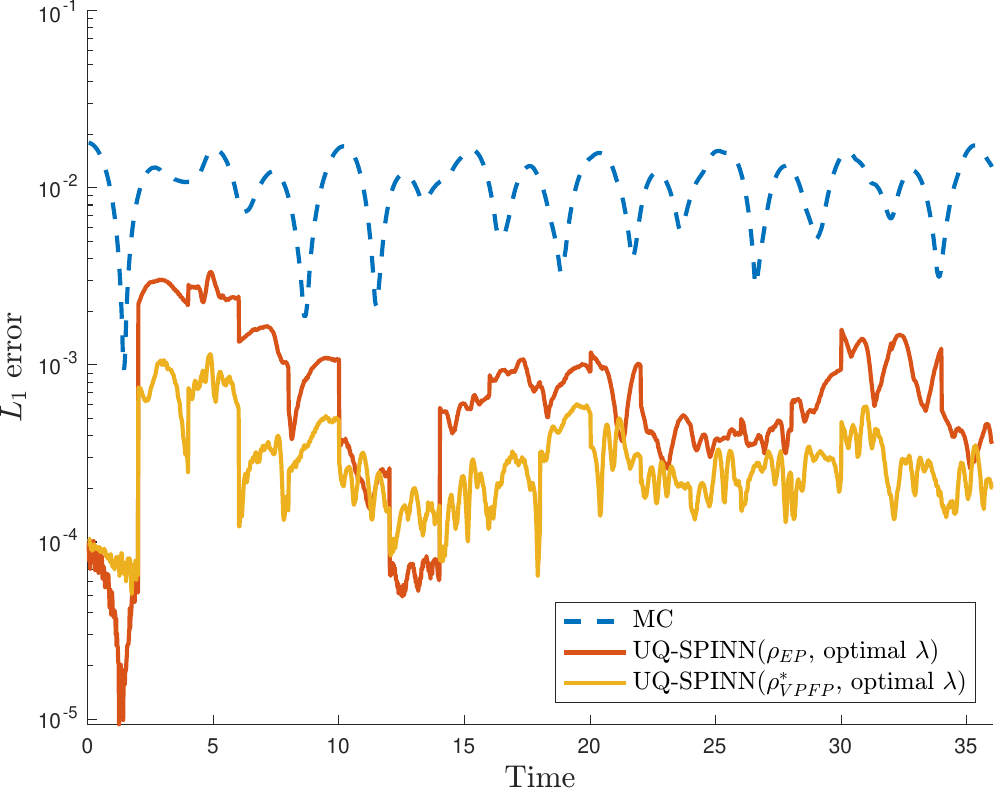}}
        {\includegraphics[width = 0.45 \textwidth, trim=0 0 0 0,clip]{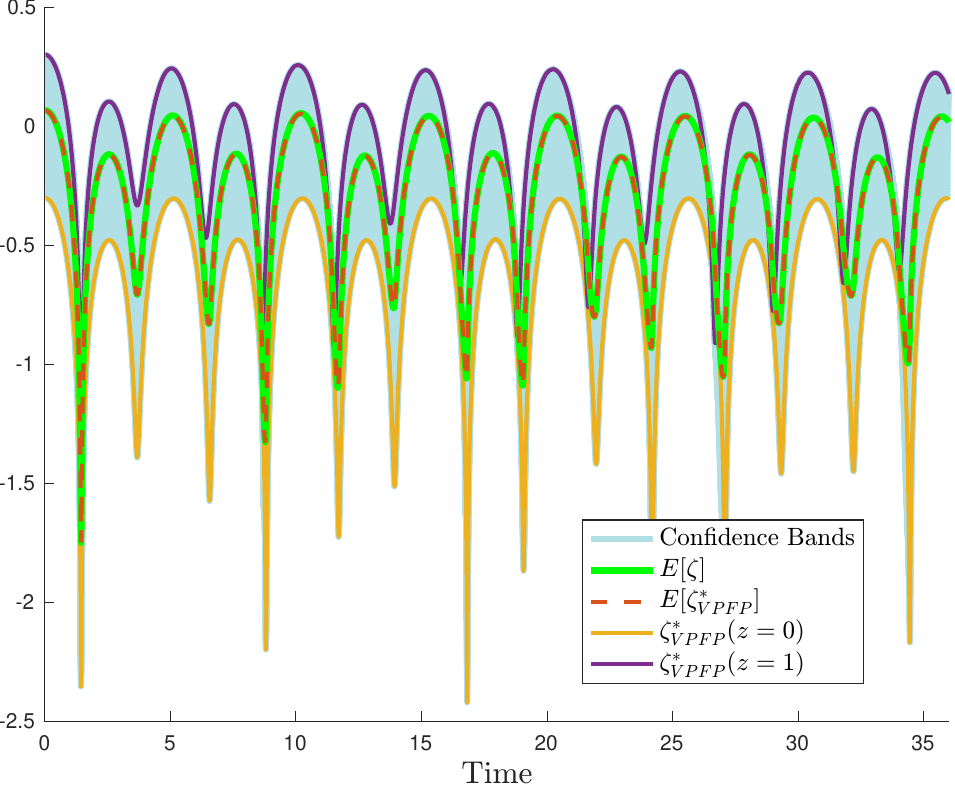}}}
        \mbox{
        {\includegraphics[width = 0.45 \textwidth, trim=0 0 0 0,clip]{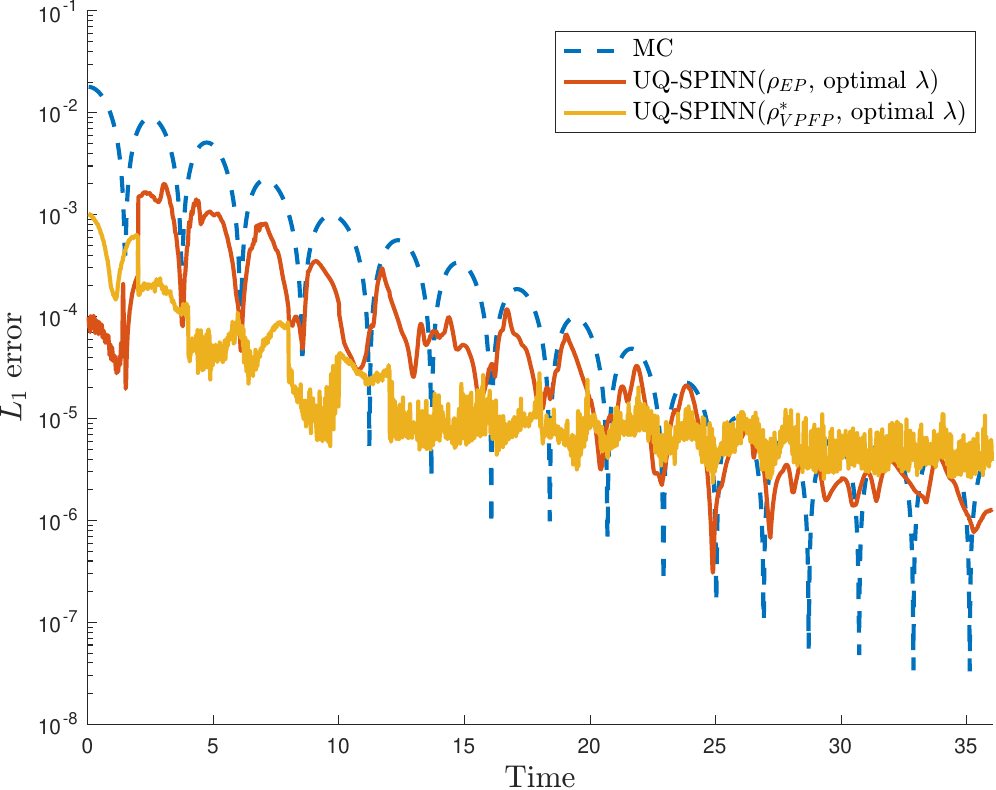}}
        {\includegraphics[width = 0.45 \textwidth, trim=0 0 0 0,clip]{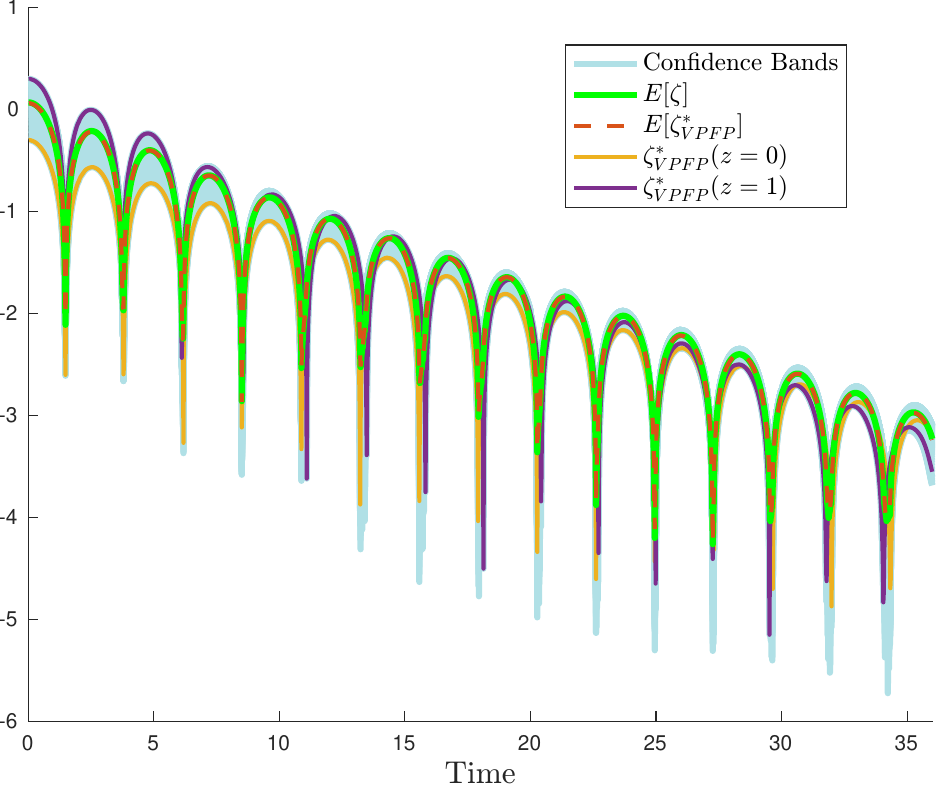}}}
        \mbox{
        {\includegraphics[width = 0.45 \textwidth, trim=0 0 0 0,clip]{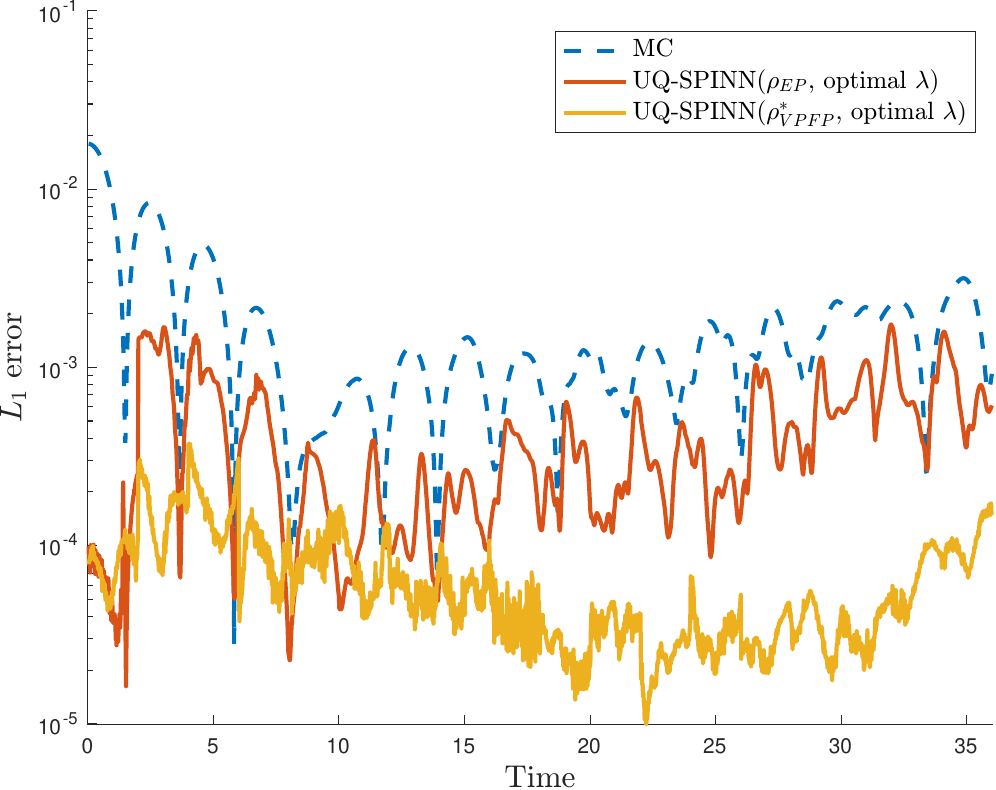}}
        {\includegraphics[width = 0.45 \textwidth, trim=0 0 0 0,clip]{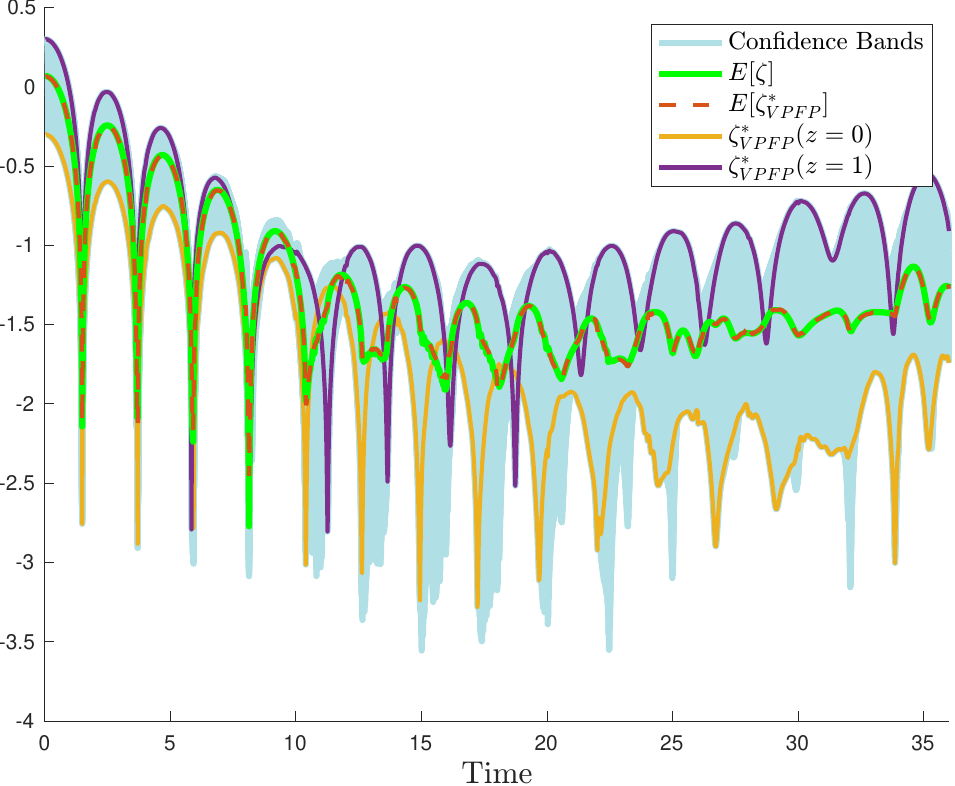}}}
        \caption{\sf Results for Example \ref{exam3}. The sample sizes used to compute the expected value and to construct the control variate are $K = 15$, $L = 20000$ for calibrated VPFP and EP, respectively. Left: $L_1$ error of expectation $\mathbb{E}[\rho]$; Right: expectation $\mathbb{E}[\zeta]$ and confidence bands; Top: $\varepsilon = 1e-4$; Middle: $\varepsilon = 1$; Bottom: $\varepsilon = 1e+6$.} 
        \label{NLD}
    \end{center}
\end{figure}

\begin{remark}[Surrogate error vs PINN error]
\label{NLDRemark}
Figure~\ref{NLDAdd} reports the nonlinear Landau damping results for \(\varepsilon=1\), where Monte Carlo estimators are enhanced via control variates constructed from the EP model and a calibrated VPFP model. We compare two implementations of the control variates: (i) UQ-SPINN-based surrogates and (ii) a deterministic solver (DS). For UQ-SPINN, the control-variate means are estimated using \(L=20000\) samples. For DS, estimating these means by brute-force random sampling would be prohibitively expensive; instead, we approximate the expectation using Gauss quadrature in the random dimension \(z\). While this quadrature strategy is accurate for low-dimensional randomness, it suffers from the curse of dimensionality as the dimension of \(z\) increases.

In the left panel of Fig.~\ref{NLDAdd} (EP-based control variates), DS is more accurate at early times, primarily because the control-variate mean is computed more accurately via quadrature. For \(t>6\), however, UQ-SPINN and DS yield nearly indistinguishable errors. In this regime, the total error is dominated by the model discrepancy between EP and VPL, which substantially limits the achievable variance reduction.

In the right panel of Fig.~\ref{NLDAdd} (calibrated-VPFP-based control variates), we distinguish three time intervals, \([0,9]\cup[9,25]\cup[25,36]\). Over \([0,9]\), DS again achieves smaller errors, since its control-variate mean is obtained by collocation/quadrature rather than by noisy sampling. Over \([9,25]\) and \([25,36]\), however, the mismatch between the calibrated VPFP model and VPL grows over time in DS. Moreover, DS propagates exclusively from the initial condition at \(t=0\), so the resulting model error accumulates and progressively degrades variance reduction.

In contrast, UQ-SPINN employs a windowed training strategy: within each time window, both the initial and interior training data are provided by the VPL-based Monte Carlo solution. Since each window spans only \(\Delta t=2\), model discrepancies do not amplify as severely as in long-horizon DS simulations, which explains the superior performance of UQ-SPINN over \([9,25]\). Over \([25,36]\), UQ-SPINN no longer reduces the Monte Carlo error, and the error level remains comparable to that observed in \([9,25]\). Because windowed training largely suppresses accumulated model-error growth, the remaining error is primarily attributable to approximation error in the neural-network surrogate.
\end{remark}

\begin{figure}[tb]
    \begin{center}
        \mbox{
        {\includegraphics[width = 0.45 \textwidth, trim=0 0 0 0,clip]{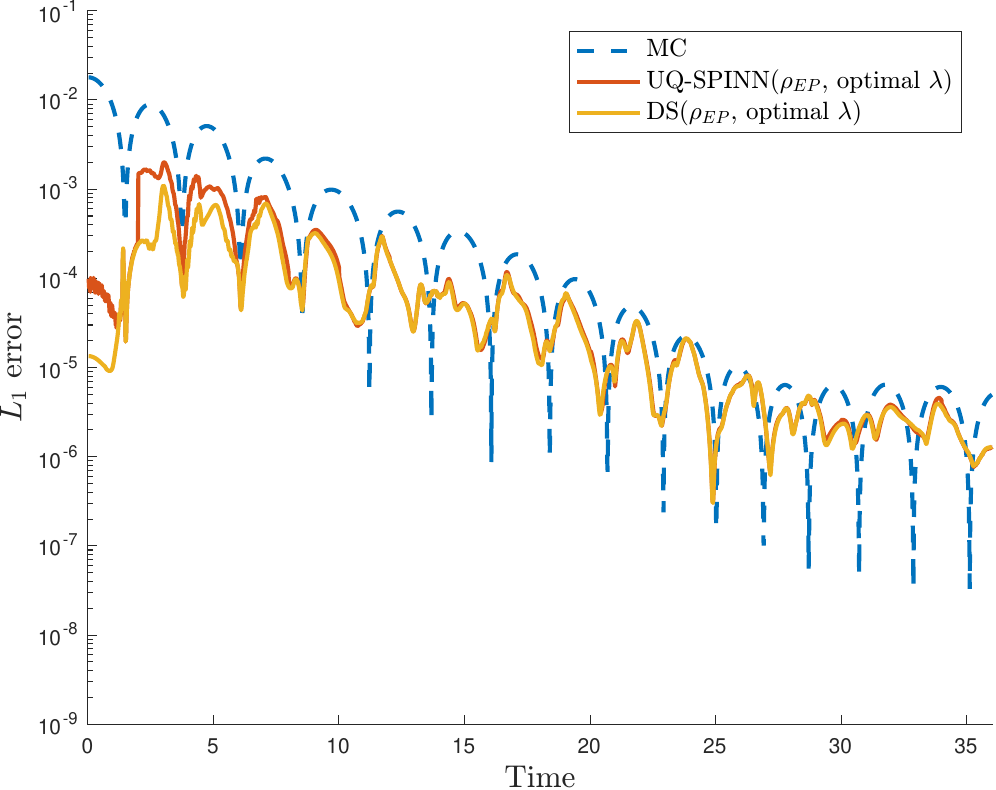}}
        {\includegraphics[width = 0.45 \textwidth, trim=0 0 0 0,clip]{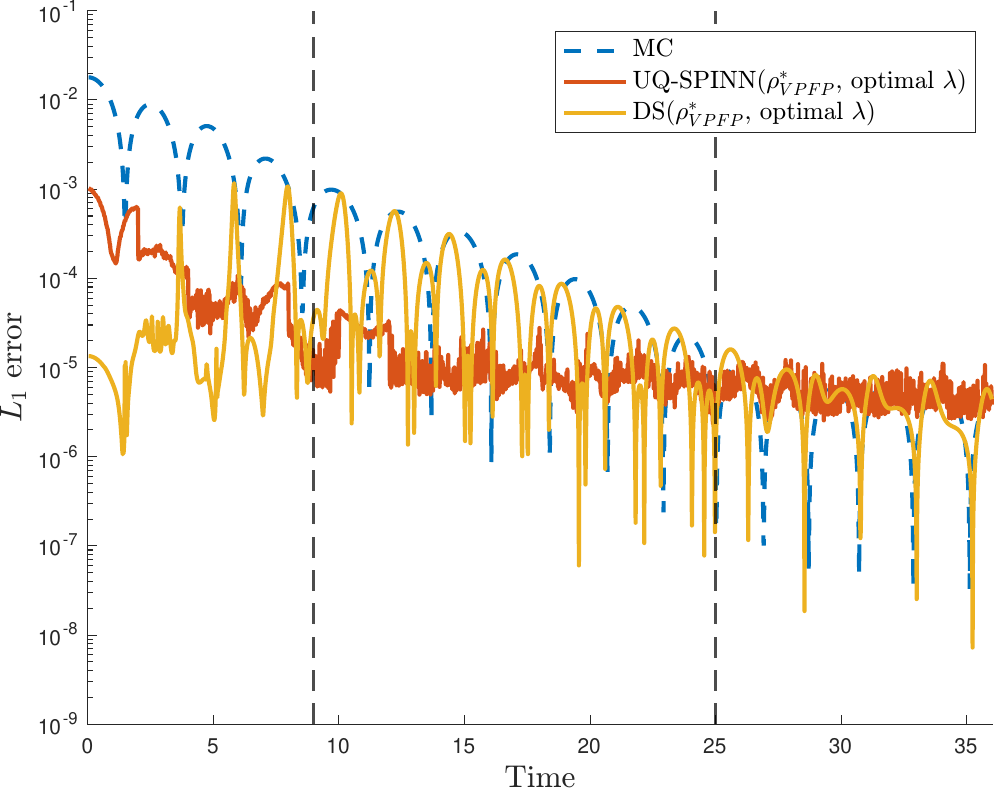}}}
        \caption{\sf Results for Example \ref{exam3}. $L_1$ error of expectation $\mathbb{E}[\rho]$. Left: based on EP model; Right: based on calibrated VPFP model.} 
        \label{NLDAdd}
    \end{center}
\end{figure}

\subsection{Two stream instability}
\label{exam4}
In the last case, the initial data are
\begin{equation*}
\begin{aligned}
  f_0(z, x, \mathbf{v})
  = &\frac{1}{4\pi T_0}\!\left(1 + (4 + 2z)\,\alpha \cos(kx)\right) \\
    &\times \left(\exp\!\left(-\frac{\left(v_x - d\right)^2 + v_y^2}{2T_0}\right) + \exp\!\left(-\frac{\left(v_x + d\right)^2 + v_y^2}{2T_0}\right)\right),
\end{aligned}
\end{equation*}
with $\alpha = 0.01$, $d = 1.3$, $T_0 = 0.3$, and $k=2/13$.
We set $\Omega_x=[0,13\pi]$ and other settings are the same as in Example~\ref{exam3}.
Figs.~\ref{TSI01} and \ref{TSIinf} display the expectation of projection $\mathcal{P}_{{VPFP}}^*$ of the distribution $f_{{VPFP}}^*$ onto the $(x,v_x)$ plane at various times and Knudsen numbers. 
Under strong collisionality (Fig.~\ref{TSI01}, top), the distribution function relaxes rapidly toward a Maxwellian. 
Under mild collisionality (Fig.~\ref{TSI01}, bottom), a small central hole emerges at \(t=1.4\); by \(t=8\) collisional effects dominate and \(f\) quickly relaxes to a Maxwellian state. 
In the collisionless case (Fig.~\ref{TSIinf}), instability develops.
\begin{figure}[tb]
    \begin{center}
        \mbox{
        {\includegraphics[width = 0.32 \textwidth, trim=0 0 0 0,clip]{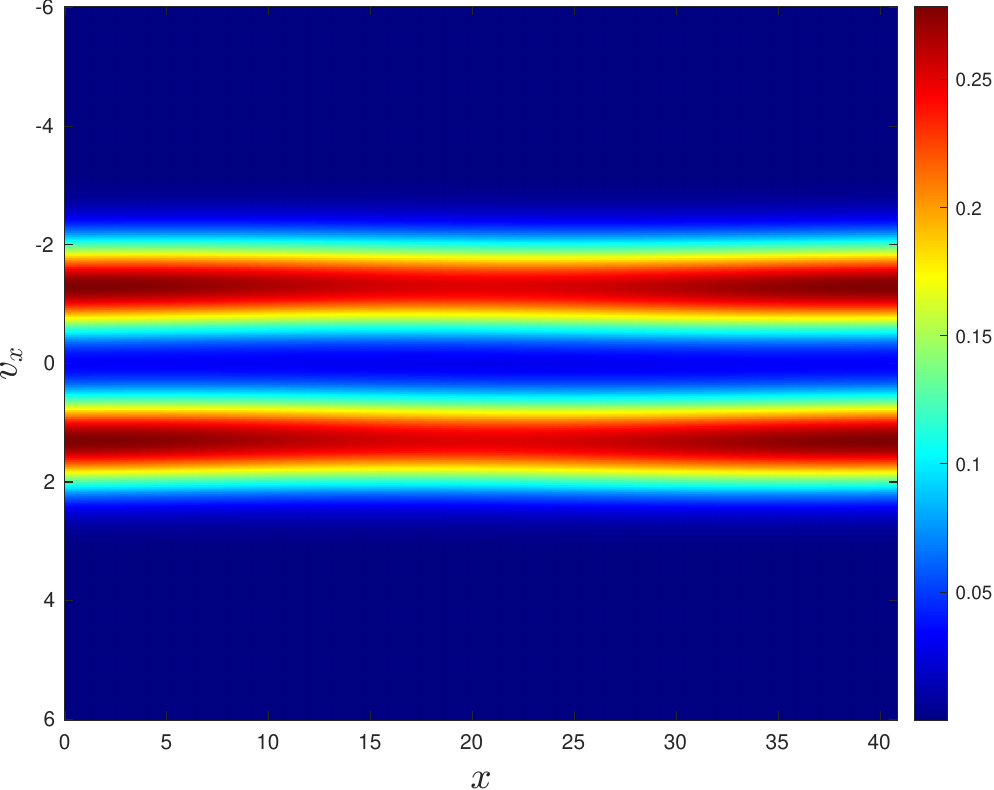}}
        {\includegraphics[width = 0.32 \textwidth, trim=0 0 0 0,clip]{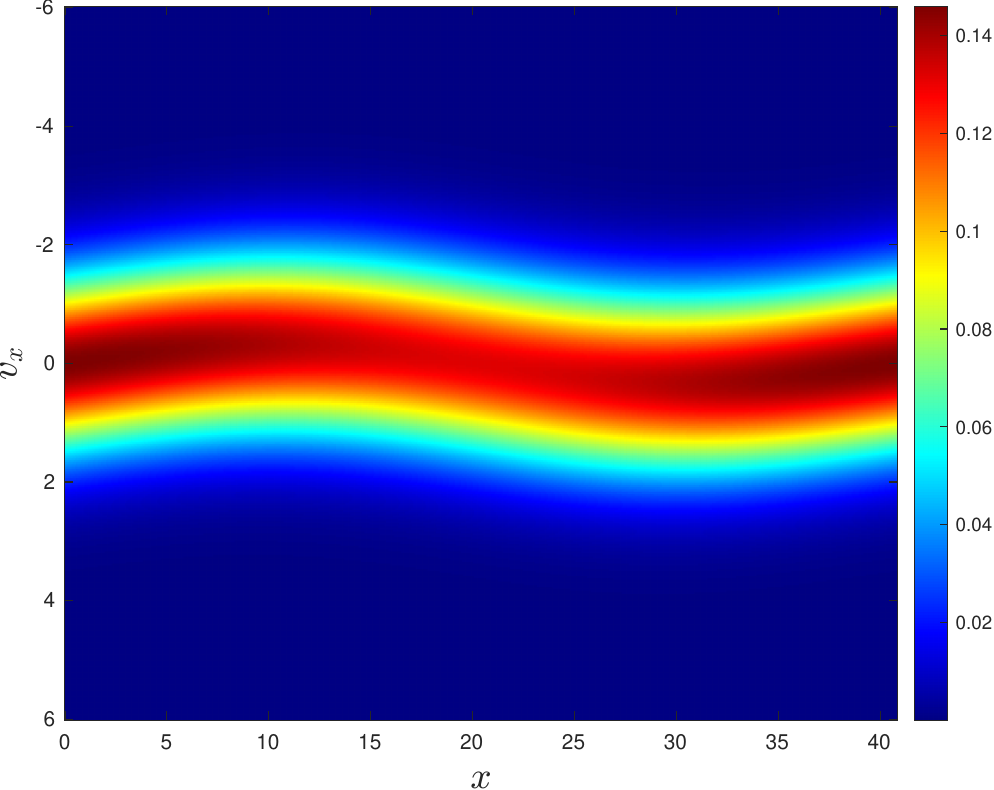}}
        {\includegraphics[width = 0.32 \textwidth, trim=0 0 0 0,clip]{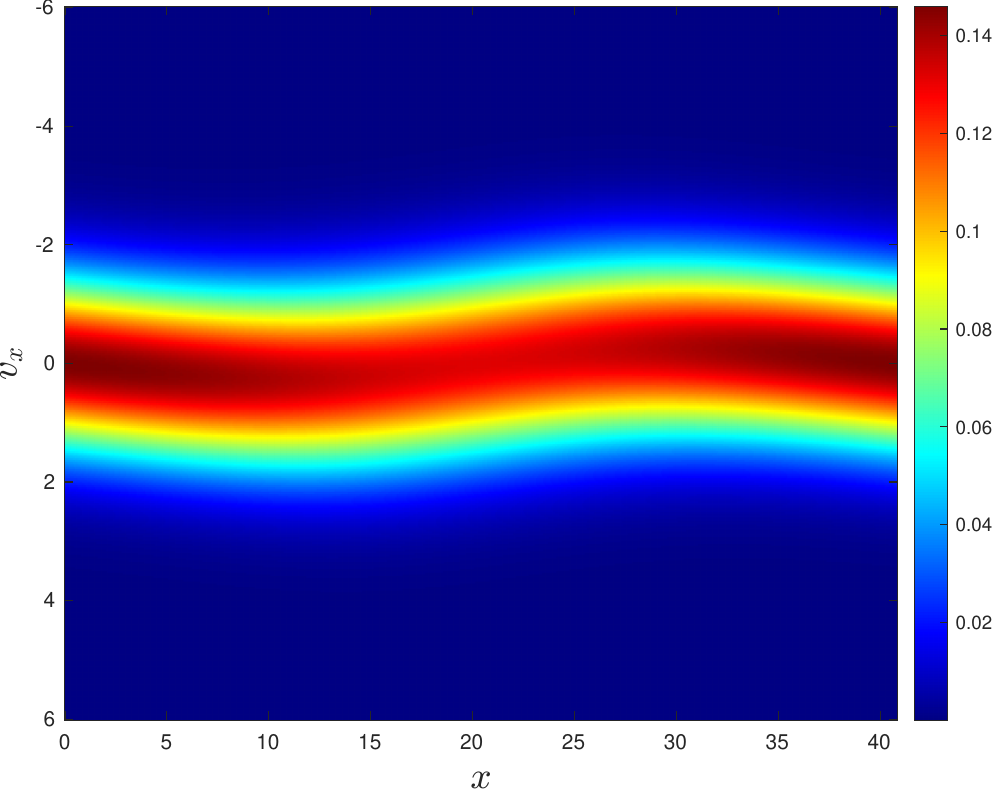}}}
        \mbox{
        {\includegraphics[width = 0.32 \textwidth, trim=0 0 0 0,clip]{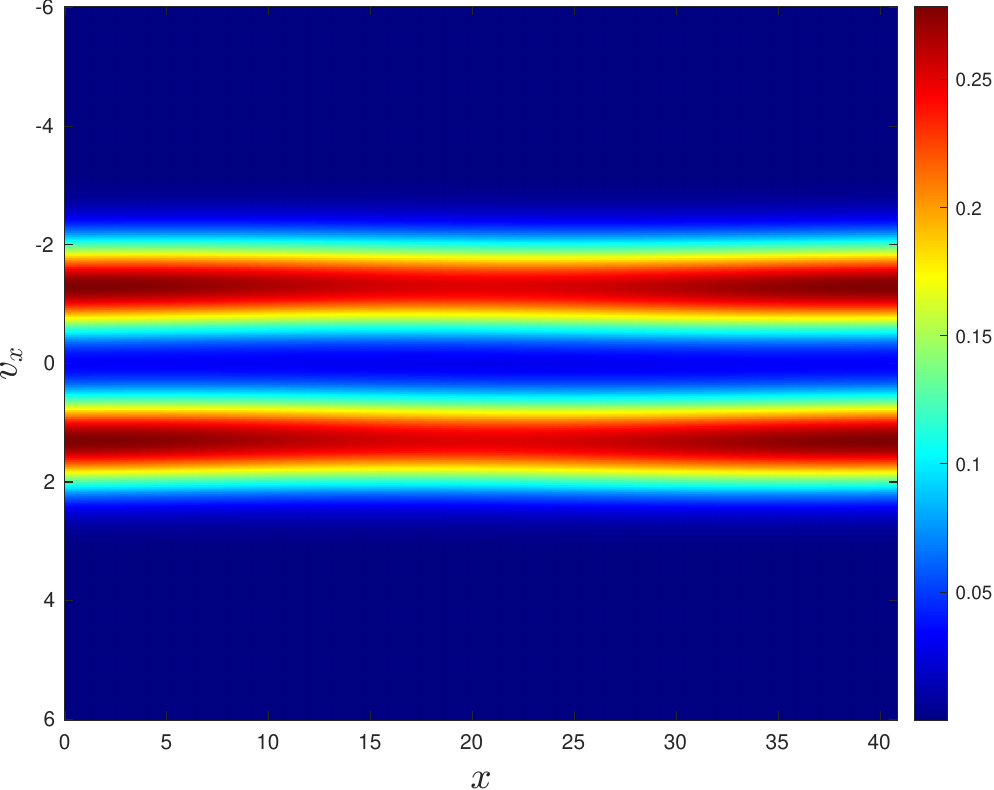}}
        {\includegraphics[width = 0.32 \textwidth, trim=0 0 0 0,clip]{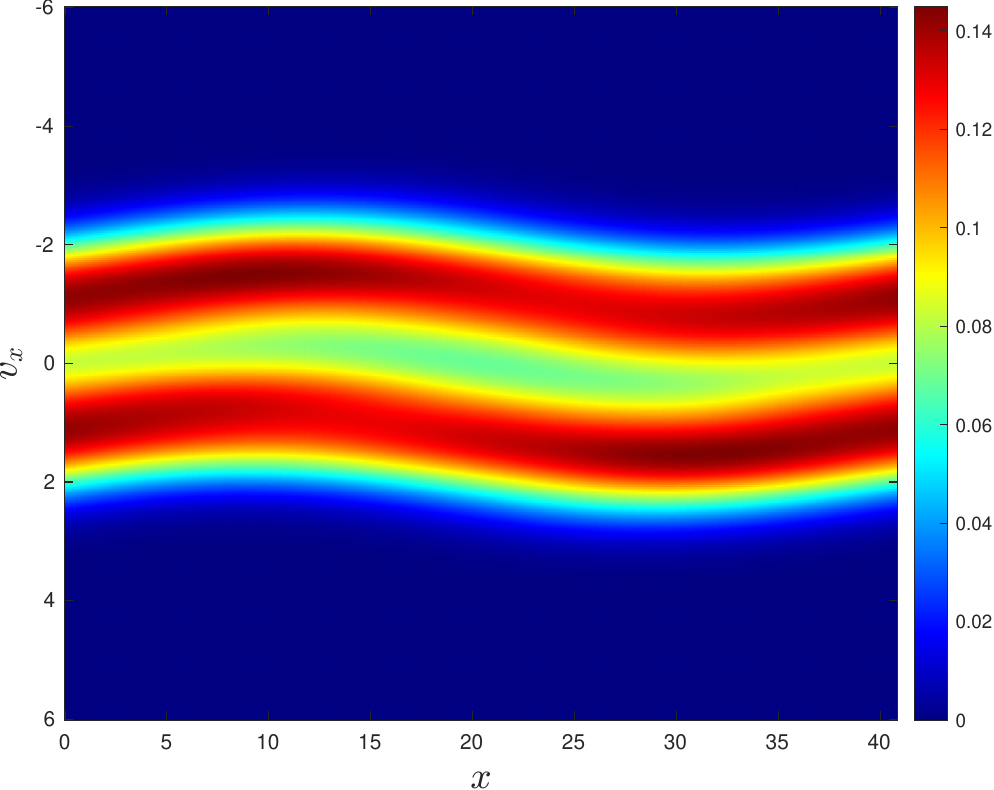}}
        {\includegraphics[width = 0.32 \textwidth, trim=0 0 0 0,clip]{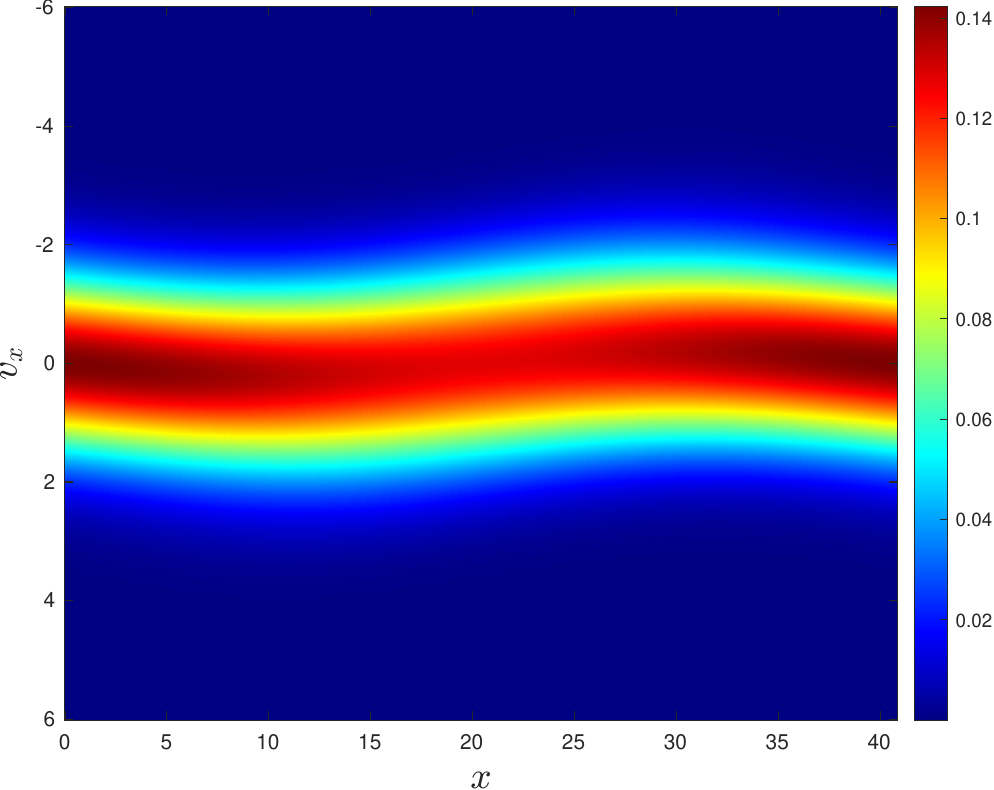}}}
        
        \caption{\sf Results for Example \ref{exam4}. Expectation of projections $\mathcal{P}_{VPFP}^*$ at $t = 0, 1.4,$ and $8$; Top: $\varepsilon = 1e-4$; Bottom: $\varepsilon = 1$.} 
        \label{TSI01}
    \end{center}
\end{figure}

\begin{figure}[tb]
    \begin{center}
        \mbox{
        {\includegraphics[width = 0.32 \textwidth, trim=0 0 0 0,clip]{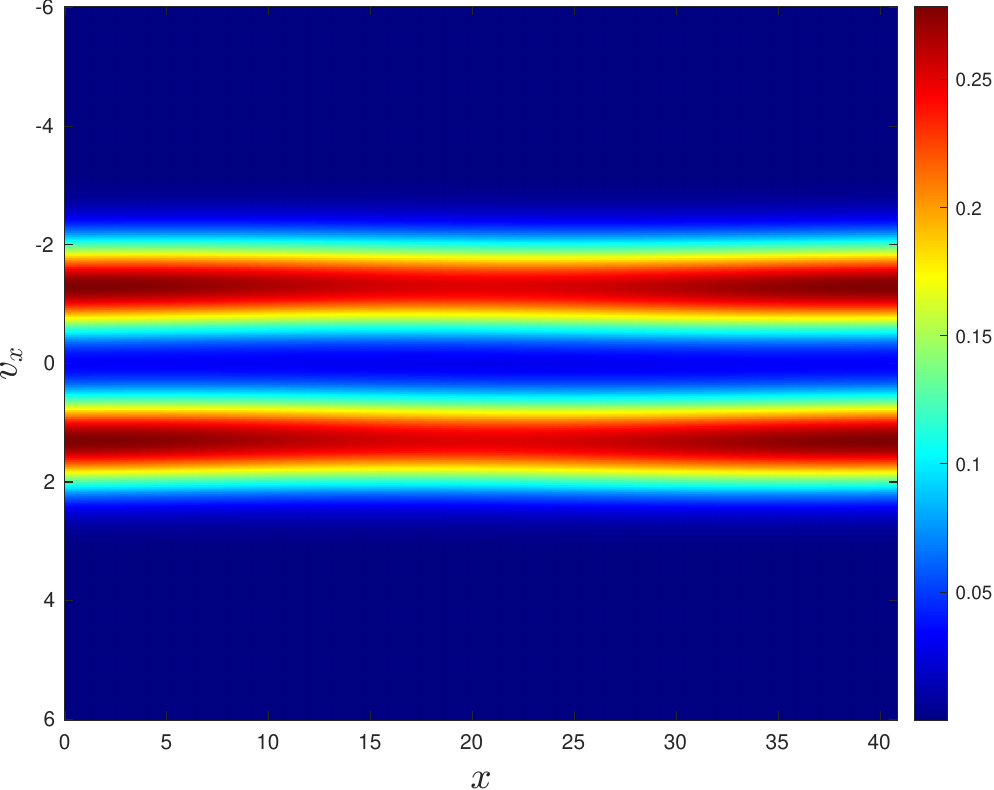}}
        {\includegraphics[width = 0.32 \textwidth, trim=0 0 0 0,clip]{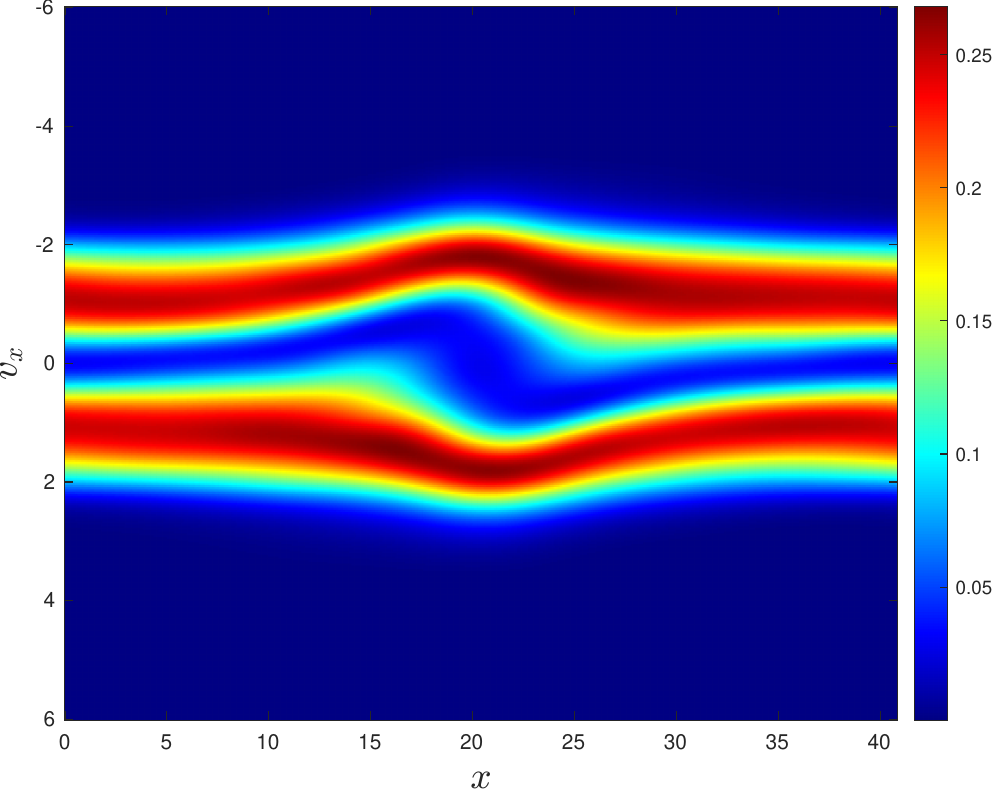}}
        {\includegraphics[width = 0.32 \textwidth, trim=0 0 0 0,clip]{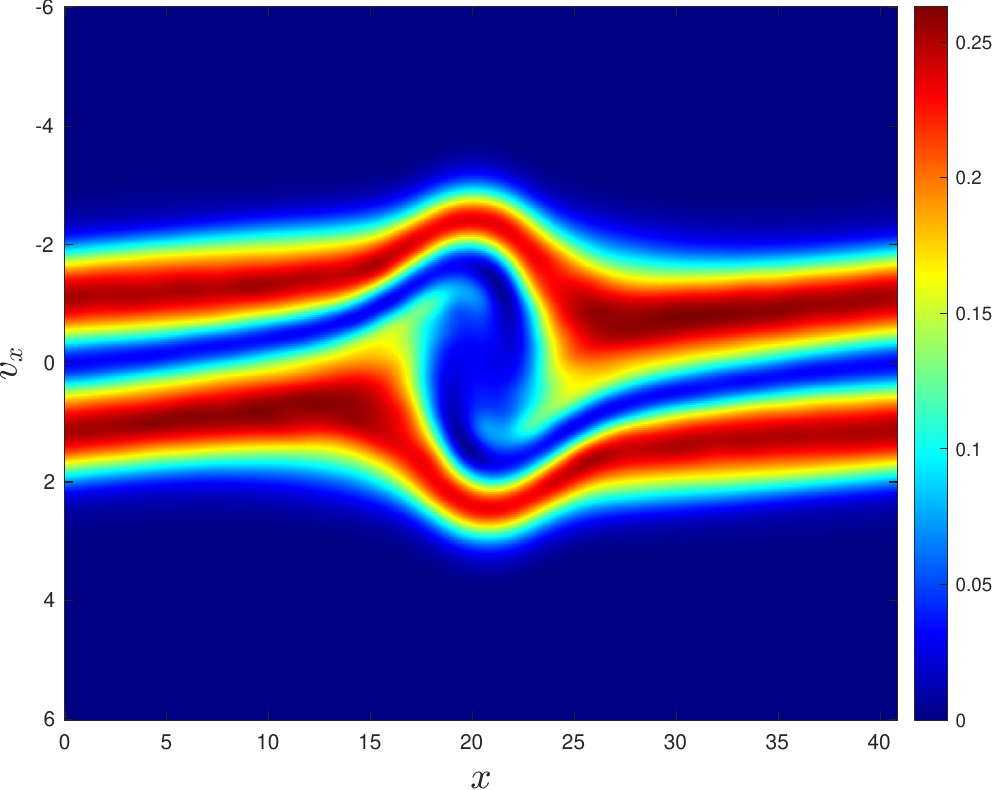}}}
        \caption{\sf Results for Example \ref{exam4}. Expectation of projections $\mathcal{P}_{VPFP}^*$ at $t = 0, 15,$ and $20$ for $\varepsilon = 1e+6$.} 
        \label{TSIinf}
    \end{center}
\end{figure}

Fig.~\ref{TSIError} compares the MC error and the variance-reduced error using $K=15$ samples (for estimating the expectation and training UQ-SPINN) and $L=20000$ samples (for constructing the control variates).
We set \(\alpha\) as in the linear Landau–damping benchmark, and the nonlinearity is mild. 
Consequently, the variance-reduction behavior of UQ-SPINN\((\rho_{{EP}})\) across Knudsen numbers mirrors the linear case: at \(\varepsilon=1e-4\) its MC-error reduction is comparable to that of UQ-SPINN\((\rho^{*}_{{VPFP}})\) (about two orders of magnitude), and although its performance remains substantial as the Knudsen number increases, it exhibits larger fluctuations. 
However, for UQ-SPINN\((\rho^{*}_{{VPFP}})\), the MC-error reduction degrades with increasing Knudsen number. 
Unlike the linear/nonlinear Landau–damping tests—which remain near steady state across Knudsen numbers—the present two-stream instability does not reach steady state in velocity space and features two bubble-like structures; approximating such shapes requires more training time and computational effort. 
As a result, at \(\varepsilon=1\) (early times) and \(\varepsilon=1e+6\) (throughout the horizon), UQ-SPINN\((\rho^{*}_{{VPFP}})\) underperforms UQ-SPINN\((\rho_{{EP}})\). 
Nevertheless, it still yields a substantial MC-error reduction and, importantly, supplies the projection \(\mathcal{P}^{*}_{{VPFP}}\), which is unavailable to UQ-SPINN\((\rho_{{EP}})\).
Finally, the EP-based and calibrated VPFP-based surrogates have already produced enough samples to realize variance reduction. 
Under the accuracy limits of the neural networks, multiple VRMC does not further decrease the variance and may even degrade accuracy because of finite-precision issues in the weights; accordingly, we omit those results.

In Table~\ref{tab:runtime}, we summarize runtimes for the different methods. 
Although the neural network entails a training cost, this overhead is quickly amortized by its fast inference, particularly when many samples are required.
Assuming the same 20000 samples, the ratio of wall-clock times for the EP-based UQ-SPINN, the calibrated-VPFP-based UQ-SPINN, and the deterministic solver for EP, calibrated VPFP, and VPL is approximately $1 : 6.27 : 2.02: 6.72 \times 10^{3} : 9.32 \times 10^{3}$.

\begin{figure}[tb]
    \begin{center}
        \mbox{
        {\includegraphics[width = 0.45 \textwidth, trim=0 0 0 0,clip]{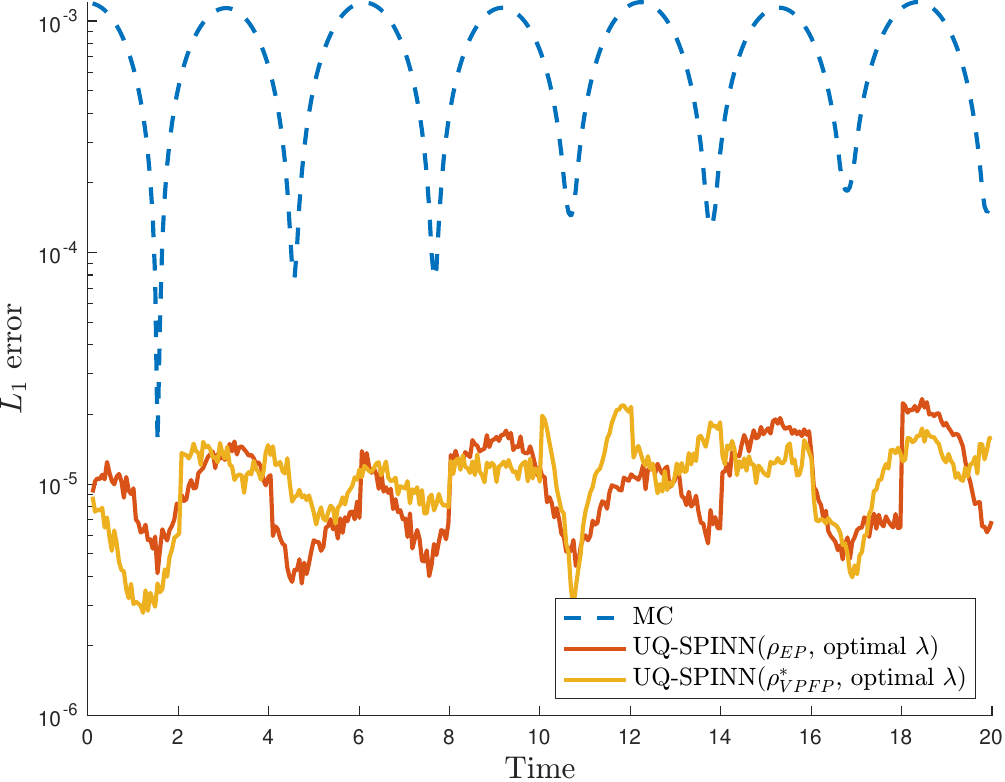}}
        {\includegraphics[width = 0.45 \textwidth, trim=0 0 0 0,clip]{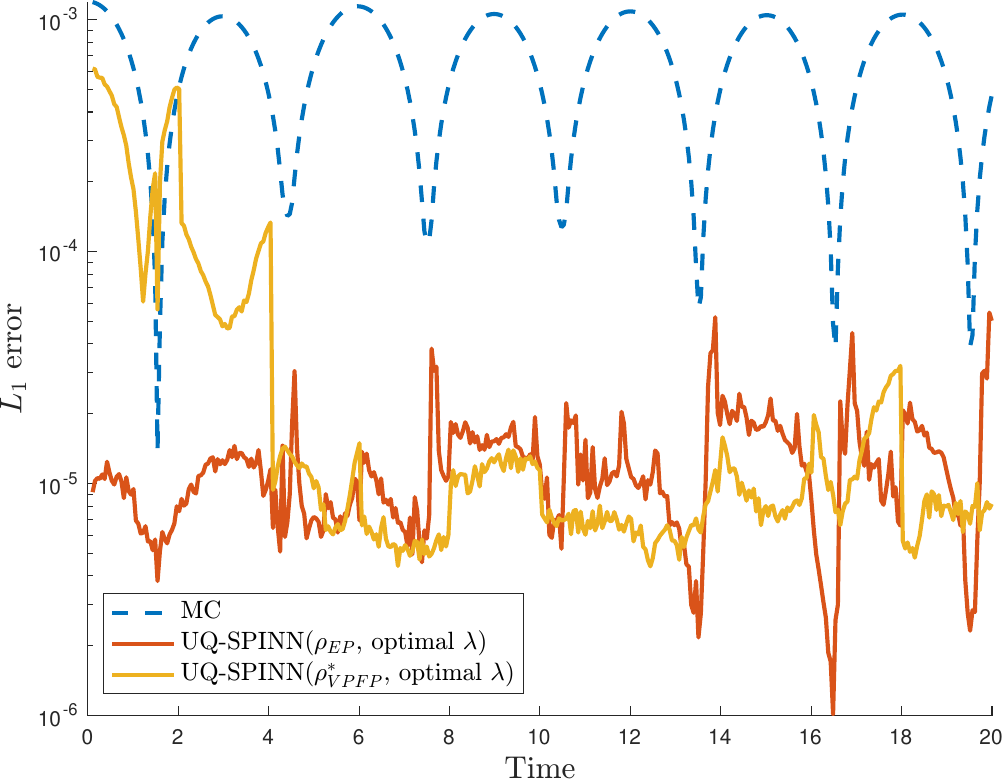}}}
        \mbox{
        {\includegraphics[width = 0.45 \textwidth, trim=0 0 0 0,clip]{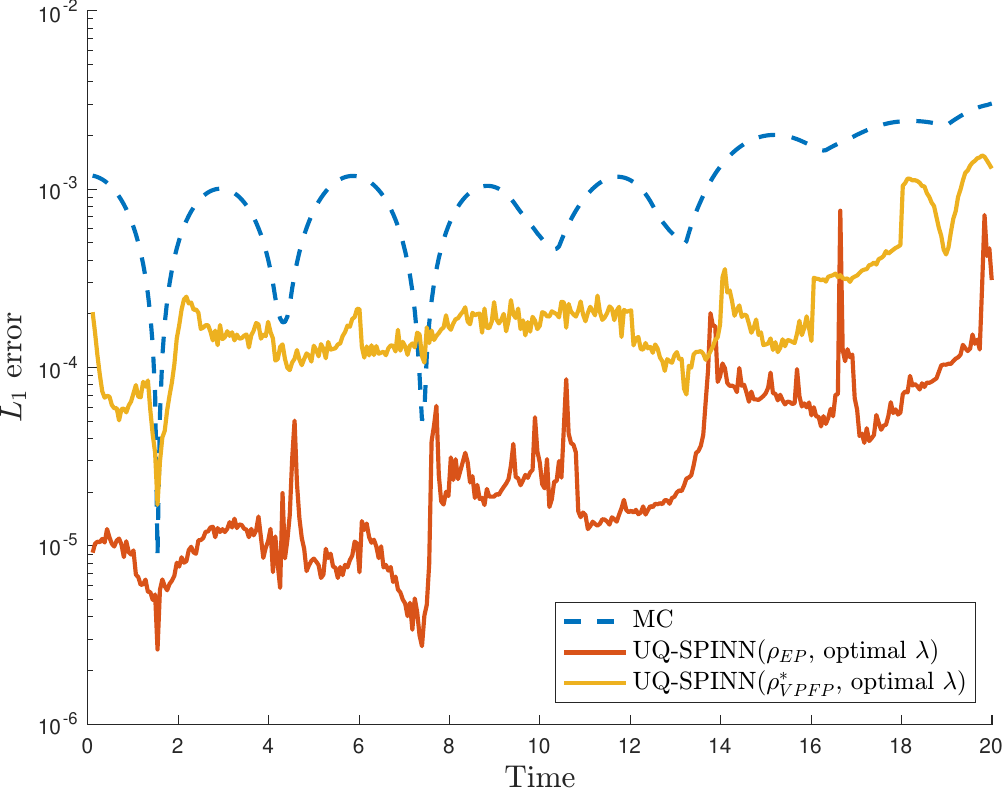}}}
        \caption{\sf $L_1$ error of expectation $\mathbb{E}[\rho]$ for Example \ref{exam4}. The sample sizes used to compute the expected value and to construct the control variate are $K = 15$, $L = 20000$ for calibrated VPFP and EP, respectively. Top Left: $\varepsilon = 1e-4$; Top Right: $\varepsilon = 1$; Bottom: $\varepsilon = 1e+6$.} 
        \label{TSIError}
    \end{center}
\end{figure}

\section{Conclusion}
\label{con}
This work has presented a variance-reduced Monte Carlo framework for uncertainty quantification in the Vlasov--Poisson--Landau (VPL) system, combining a high-fidelity, asymptotic-preserving solver for the VPL equation with lower-fidelity physical models based on the Vlasov--Poisson--Fokker--Planck (VPFP) and Euler--Poisson (EP) equations through optimally chosen control-variate weights. A central aspect of the methodology is the use of neural-network surrogates in conjunction with these physical low-fidelity models. Although VPFP and EP provide meaningful reduced descriptions of the VPL dynamics, generating the large number of low-fidelity realizations required by variance-reduction techniques may still be computationally demanding when each realization is obtained by directly solving the corresponding time-dependent PDE. The introduction of tensor neural networks and of an enhanced separable physics-informed neural network (SPINN) therefore provides an additional level of approximation in which the offline training cost can be amortized when many samples are needed, and low-fidelity realizations are produced through forward evaluations of trained networks rather than repeated numerical time stepping.

The proposed SPINN architecture, based on a micro--macro decomposition with an anisotropic Maxwellian background, is designed to preserve physical structure and asymptotic consistency. The VPFP surrogate is further calibrated using a limited set of high-fidelity VPL data and trained through a windowed strategy. This windowed training is essential to avoid long-horizon error accumulation and yields an accurate neural-network representation of the surrogate dynamics over extended times. Importantly, learning the VPFP surrogate via UQ-SPINN also provides a decisive computational advantage over direct VPFP simulations: once trained, the surrogate delivers extremely fast inference (Table~\ref{tab:runtime}), enabling the evaluation of tens of thousands of realizations at negligible marginal cost.

The numerical experiments show that the most demanding regimes arise at long times, not because the neural-network approximation deteriorates, but because the correlation between the surrogate model (EP or calibrated VPFP) and the VPL dynamics may progressively weaken. In this situation, increasing the number of control-variate samples does not necessarily improve the results, since the dominant limitation becomes the intrinsic model discrepancy rather than sampling error. This observation also clarifies why, in our neural-network setting, we did not carry out a separate computational-cost optimization by varying the number of samples between different learned control variates: for the relevant ranges of samples, inference is negligible compared with training, and beyond a threshold the performance saturates due to model-discrepancy effects rather than insufficient sampling.

Importantly, the variance-reduced Monte Carlo estimator remains robust under these conditions. When the correlation between low- and high-fidelity models deteriorates, the optimal control-variate weights automatically decrease and the estimator smoothly degenerates to standard Monte Carlo without loss of consistency or stability. In this sense, the variance-reduction framework provides a fail-safe mechanism: even when a surrogate model ceases to be effective as a control variate in a given long-time regime, the estimator remains unbiased and stable, reverting seamlessly to standard Monte Carlo without user intervention.

Overall, the numerical results confirm that substantial variance reduction can be achieved whenever sufficient correlation with VPL is preserved. Accurate statistical estimates are obtained with significantly fewer high-fidelity samples and reduced effective computational cost, while maintaining robustness across stochastic dimensions and physical regimes. In particular, the learned VPFP surrogate combines strong practical efficiency (very low inference cost) with improved long-time behavior through windowed training, and serves as an effective control variate over wide parameter and time ranges.

\bibliographystyle{siamplain}
\bibliography{references}
\end{document}